\documentclass{article}
\usepackage[utf8]{inputenc}
\usepackage[T1]{fontenc}    

\usepackage[margin=1in]{geometry}
\usepackage[numbers]{natbib}

\usepackage{microtype}
\usepackage{graphicx,epstopdf}
\usepackage{url}            

\usepackage{threeparttable}
\usepackage{booktabs}
\usepackage{multirow}
\usepackage{array}
\usepackage{makecell}
\usepackage{colortbl}
\usepackage{tabularray}

\usepackage{amsmath}
\usepackage{amsthm}
\usepackage{amssymb}
\usepackage{amsfonts}
\usepackage{mathtools}
\usepackage{physics}
\usepackage{bm}
\usepackage{dsfont}
\usepackage{pifont}
\usepackage{wasysym}
\mathtoolsset{showonlyrefs=true}

\usepackage{xcolor}
\usepackage[table]{xcolor}
\definecolor{lightblue}{RGB}{230,246,253}
\usepackage{tocvsec2}
\usepackage{titletoc}
\usepackage{xspace}
\usepackage{enumitem}
\usepackage{wrapfig}
\usepackage{float}
\usepackage{multicol}
\usepackage{subcaption}
\usepackage{diagbox}

\usepackage{algorithm}
\usepackage{algpseudocode}
\usepackage[skins,listings]{tcolorbox}
\tcbuselibrary{breakable}
\usepackage{quiver}

\restylefloat{algorithm} 
\allowdisplaybreaks

\newtheorem{definition}{Definition}
\newtheorem{assumption}{Assumption}
\newtheorem{theorem}{Theorem}
\newtheorem{remark}{Remark}
\newtheorem{lemma}{Lemma}

\newtheorem{corollary}{Corollary}

\newcommand{\EE}{\mathbb{E}}

\newcommand{\PP}{\mathbb{P}}

\newcommand{\R}{\mathbb{R}}

\definecolor{drj}{RGB}{67,151,143}

\definecolor{lyx}{RGB}{221,160,221}

\definecolor{gray}{RGB}{169,169,169}

\newcommand{\softmax}{\text{softmax}}

\newcommand{\diag}{\text{diag}}

\newcommand{\Rot}[1]{\operatorname{Rot}(#1)}
\renewcommand{\abs}{\text{abs}}

\renewcommand{\paragraph}[1]{\par\medskip\noindent\textbf{#1.}}

\providecommand{\inner}[2]{\langle #1, #2 \rangle}
\newcommand{\OmegaSpike}{\Omega_{\tau}}
\newcommand{\OmegaBg}{\Omega_{\tau}^{c}}

\newtcblisting{tikzexample}{
    sidebyside, center lower, bicolor, colbacklower=white,
    sharp corners, frame engine=empty
}

\newtcolorbox{limbox}{
  enhanced, breakable, colback=orange!5, frame hidden, boxrule=0pt,
  borderline west={1.2pt}{0pt}{orange!55!black}, sharp corners,
  left=6pt, right=2pt, top=3pt, bottom=3pt, before skip=5pt, after skip=5pt
}
\newtcolorbox{limbox_blue}{
  enhanced, breakable, colback=blue!5, frame hidden, boxrule=0pt,
  borderline west={1.2pt}{0pt}{blue!55!black}, sharp corners,
  left=6pt, right=2pt, top=3pt, bottom=3pt, before skip=5pt, after skip=5pt
}

\usepackage[pdfusetitle,colorlinks,linkcolor=blue,filecolor=blue,citecolor=black,urlcolor=blue]{hyperref}
\hypersetup{colorlinks,urlcolor=blue}
\title{RoPeSLR: 3D RoPE-driven Sparse-LowRank Attention for Efficient Diffusion Transformers}

\author{
  Yuxi Liu\texorpdfstring{$^*$}{*} \\Peking University\\
  \texttt{yuxiliu666@stu.pku.edu.cn} \\
   \and
  Zekun Zhang\texorpdfstring{$^*$}{*} \\
University of Electronic Science and Technology of China \\
\texttt{2023090909020@std.uestc.edu.cn} \\
  \and
  Yixiang Cai\texorpdfstring{$^*$}{*} \\
Beijing University of Posts and Telecommunications \\
\texttt{caiyixiang@bupt.edu.cn}\\
  \and
    Renjia Deng \\
    Peking University \\
\texttt{2501210078@stu.pku.edu.cn}\\
\and
    Yutong He \\
    Peking University \\
\texttt{yutonghe@pku.edu.cn}\\
\and
    Kun Yuan\texorpdfstring{$^\dagger$}{†} \\
    Peking University \\
    \texttt{kunyuan@pku.edu.cn}
}

\allowdisplaybreaks

\begin{document}
\maketitle
\def\thefootnote{$^*$}\footnotetext{Equal contribution.}
\def\thefootnote{$^\dagger$}\footnotetext{Corresponding author.}
\setcounter{footnote}{0}
\renewcommand{\thefootnote}{\arabic{footnote}}
\setlength{\parindent}{0pt}

\vspace{-1.1cm}
\begin{figure}[H]
    \centering
    \includegraphics[width=0.9\textwidth]{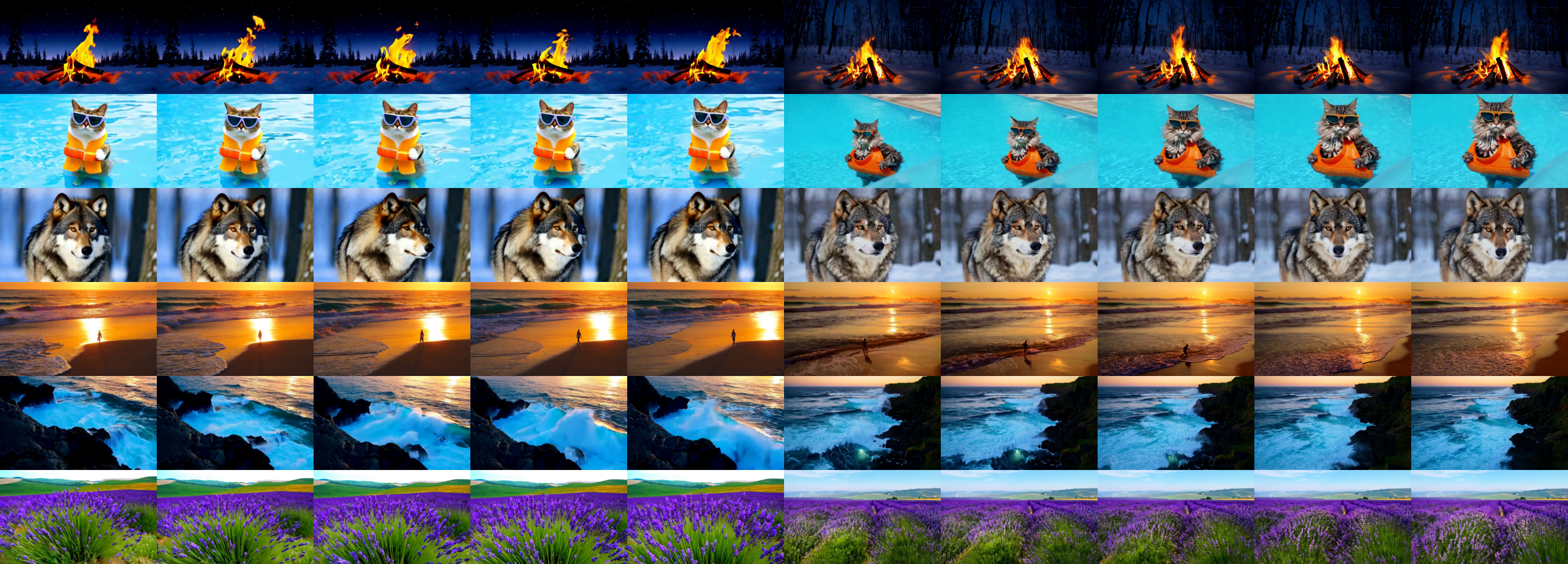}
    \caption{Video generation examples on Wan2.1-T2V-1.3B-480P. Each row shows the full attention baseline sequence followed by the corresponding RoPeSLR output generated at 90\% sparsity. RoPeSLR maintains generation quality on par with full attention even at this high sparsity.}
    \label{fig:appendix_video_samples}
\end{figure}
\vspace{-0.2cm}

\begin{abstract}
Diffusion Transformers (DiTs) have revolutionized high-fidelity video generation, yet their $\mathcal{O}(L^2)$ attention complexity poses a formidable bottleneck for long-sequence synthesis. While recent sparse-linear attention hybrids aim to mitigate this, their performance severely degrades at extreme sparsity due to the "RoPE Dilemma": standard linear attention fails to preserve the orthogonal relative-position structure of 3D Rotary Position Embeddings (RoPE), neutralizing vital distance awareness. To address this, we propose \textbf{RoPeSLR}, a 3D RoPE-guided Sparse-LowRank attention framework. We establish that under empirically validated assumptions, the DiT attention manifold admits a decoupling into a high-frequency semantic spike set (bounded by $\mathcal{O}(L^{3/2})$ sparsity) and an extreme low-rank ($\mathcal{O}(d_h \log L)$) background continuum. Guided by this structural prior, RoPeSLR eschews standard linear attention for a head-wise low-rank parameterization equipped with a learnable 3D Absolute Positional Embedding (PE) injection, seamlessly synthesizing long-range relative distance decay. By guaranteeing sub-quadratic sparsity and sub-linear rank growth, RoPeSLR is exceptionally suited for scaling to ultra-long video inference. Extensive evaluations validate this scalable superiority: at 90\% sparsity, RoPeSLR achieves up to $10\times$ fewer FLOPs on Wan2.1-1.3B and delivers a $2.26\times$ end-to-end inference speedup on the ultra-long 100K+ token sequences of HunyuanVideo-13B, all while maintaining near-lossless generation fidelity (less than 1.3\% average VBench degradation).
\end{abstract}

\vspace{-0.2cm}
\section{Introduction}
Diffusion Transformers (DiTs) \cite{peebles2023scalable} have revolutionized high-fidelity video generation, emerging as the dominant architecture for state-of-the-art models. However, processing the full spatiotemporal token sequence introduces a prohibitive computational bottleneck. The core of this inefficiency lies in the $\mathcal{O}(L^2)$ complexity of the mechanism, which scales quadratically with the sequence length $L$.

To mitigate this bottleneck, sparse attention mechanisms have been widely adopted to restrict queries to localized neighborhoods\cite{xi2025sparse,chen2025sparsevdit,yang2025sparse}. While efficient, pushing sparsity to extreme levels rigidly masks out global receptive fields, fundamentally truncating long-range semantic dependencies. This aggressive pruning often leads to semantic misalignment and severe degradation of fine-grained details in generated videos, presenting a strict ceiling on the allowable sparsity.

Recent advancements attempt to recover this lost global context by fusing sparse attention with a linear attention branch \cite{zhang2025sla, zhang2026sla2, fang2026salad}. While empirically successful at moderate sparsity, these sparse-plus-linear hybrids suffer from two fundamental disadvantages that bottleneck their performance at extreme sparsity: 
(1) \textbf{The RoPE Dilemma:} To maintain matrix associativity, these hybrid baselines rely on standard positive feature maps that \textbf{fail to preserve the exact orthogonal relative-position structure of 3D RoPE.} Consequently, the global aggregation of context neutralizes the explicit relative spatiotemporal offsets, causing the model to lose vital distance awareness.
(2) \textbf{Lack of Guarantees:} These architectures heavily rely on heuristic designs and lack a justification explaining \textit{why} a linear combination can faithfully approximate the complex, non-linear softmax manifold without severe distribution shifts.

This brings us to a pivotal question:
\begin{limbox}
\textbf{Question.} \emph{Does there exist an efficient attention paradigm that can closely approximate full attention at extreme sparsity, guided by theoretical structural analysis, while delivering superior generation quality and computational efficiency in downstream tasks?}
\end{limbox}

To answer this, we propose \textbf{RoPeSLR}, a 3D RoPE-guided Sparse-LowRank attention framework. Unlike prior heuristics, our design is grounded in an existence result (\textbf{Theorem \ref{thm:main_text_informal}}), which suggests that under empirically supported assumptions, the post-softmax attention matrix in DiTs admits a sparse-low-rank decomposition. We show that under the exact Fourier series expansion provided by 3D RoPE, pre-trained DiTs empirically exhibit a spectral concentration that structurally decouples the attention manifold into an $\mathcal{O}(L^{3/2})$ high-frequency semantic spike set and an extreme low-rank ($\mathcal{O}(d_h \log L)$) background continuum, where $d_h$ is the feature dimension per head.

\textbf{Crucially, we establish why this sparse-low-rank paradigm is uniquely powerful for efficient long-sequence inference in DiTs.} While sparse-low-rank designs exist in LLMs, their scaling efficiency is fundamentally amplified in DiTs: the 3D axis-splitting of RoPE accelerates spectral decay, while spatiotemporal locality yields a much tighter sparsity bound. This ensures that RoPeSLR remains exceptionally compact—maintaining extreme sparsity and a minimal rank requirement \textbf{even as sequence length $L$ scales}—while providing a nearly lossless approximation of full attention.

RoPeSLR systematically addresses the flaws of prior methods by eschewing standard linear attention in favor of a \textbf{non-linear low-rank MLP compensator}. This architectural shift is necessary to overcome the limitations of linear branches via three high-level advantages: (1) \textbf{RoPE Compatibility}: Combined with a novel learnable 3D Absolute Positional Embedding (PE) injection, our low-rank bottleneck synthesizes the relative distance decay that linear feature maps inherently destroy. (2) \textbf{Non-linear Expressivity}: It breaks free from the restrictive positive feature maps required by linear attention, unlocking robust non-linear representation capacity. (3) \textbf{Theoretically-Inspired Amortization}: While our existence theorem establishes the global low-rank manifold of the background continuum, executing explicit $\mathcal{O}(L)$ linear routing still compromises RoPE structures. Instead, we implement the low-rank compensator as an \textbf{amortized token-wise MLP proxy}. Conditioned on globally-modulated token features and explicit 3D absolute coordinates, this independent head-wise parameterization efficiently decodes the smooth, low-frequency background field. It translates our theoretical insights into a hardware-efficient framework that entirely circumvents sequence-level aggregation, harmonizing beautifully with the dynamically allocated sparsity budgets.

Our main contributions are summarized as follows:
\vspace{-0.2cm}

\begin{itemize}[leftmargin=*, itemsep=0em]
    \item \textbf{Theoretical Foundation.} We provide an analytical existence result demonstrating that the post-softmax attention matrix in DiTs exhibits a sparse-low-rank structure. We derive asymptotic bounds for both branches and prove that this decomposition is more pronounced in DiTs than in LLMs.
    \item \textbf{Synergistic Architecture \& Explicit Credit Assignment.} We introduce RoPeSLR, a 3D RoPE-driven architecture that explicitly decouples computational responsibilities. We adopt block-sparse attention (e.g., VMoBA) as a hardware-efficient engine to deliver the \textbf{primary inference speedup}. In contrast, our \textbf{core novelties}—the non-linear low-rank compensator, absolute 3D PE injection, and distribution-aware gated fusion—are specifically designed to synthesize global context and recover relative positional decay, delivering the \textbf{critical quality restoration}.
    \item \textbf{Excellent Performance on Ultra-Long Sequences.}  Extensive experiments across large-scale DiTs (e.g., Wan2.1-1.3B/14B, Hunyuan13B) demonstrate that RoPeSLR excels in the quality-efficiency trade-off, particularly as sequence length scales. It retains the hardware efficiency of extreme sparse attention—achieving up to \textbf{$10\times$ FLOPs reduction and a substantial $2.26\times$ end-to-end latency speedup} on ultra-long videos exceeding 100K tokens—while accurately restoring the fine-grained details of full attention with less than 1.3\% average VBench degradation.
    
\end{itemize}

\section{Related work}
\label{sec:Related work}
\textbf{Diffusion Transformers and Efficiency Challenges.}
Diffusion Transformers (DiTs) \cite{peebles2023scalable} have emerged as the dominant architecture for high-fidelity video generation. Earlier decoupled approaches \cite{ma2024latte,lu2023vdt,zheng2024opensora} split video into separate spatial and temporal attention branches, trading joint spatio-temporal modeling capacity for efficiency. Recent unified 3D DiTs \cite{zheng2024opensora,yang2025cogvideox,kong2024hunyuanvideo,klingteam2025klingomni} process the full token sequence jointly, achieving superior quality but suffering from prohibitive $O(N^2)$ self-attention complexity that scales poorly with increasing video resolution and length.

\textbf{Sparse Attention for Video DiTs.}
Sparse attention is the most widely adopted approach to address the quadratic complexity, leveraging the sparsity of video attention patterns where each token only attends to a small subset of relevant tokens.
Dynamic sparse attention methods adapt attention patterns to input content: Radial Attention \cite{li2025radial} introduces global energy decay-based radial receptive fields, Vmoba \cite{wu2025vmoba} proposes mixture-of-block attention for video diffusion.
Sampling-based dynamic sparse attention methods choose optimal attention patterns during inference: Sparse-vDiT \cite{chen2025sparsevdit}, Sparse VideoGen \cite{xi2025sparse} and its successor Sparse VideoGen2 \cite{yang2025sparse} explore spatial-temporal sparsity and semantic-aware permutation, while VSA \cite{zhang2025vsa} introduces fully trainable sparse attention.
Training-free sparse methods offer plug-and-play acceleration: SpargeAttention \cite{zhang2025spargeattention} and the training-free adaptive sparse attention approach \cite{xia2025trainingfree} dynamically adjust sparsity without retraining. Mixture of Distributions Attention (MOD-DiT) \cite{liu2026mixture} realizes sparse attention via online inference-time prediction of attention patterns. Re-ttention \cite{chen2025retention} achieves ultra-sparsity via attention statistical reshaping, LiteAttention \cite{shmilovich2025liteattention} focuses on temporal sparsity, and DSV \cite{tan2026dsv} exploits dynamic sparsity for training acceleration.

\textbf{Sparse-Linear Attention Fusion.}
While sparse attention delivers significant speedups, it suffers from quality degradation at high sparsity levels. To overcome this, recent works combine sparse attention with linear attention, which has linear complexity but complementary approximation properties.
SLA \cite{zhang2025sla} first proposed sparse-linear attention fusion for diffusion transformers, followed by SLA2 \cite{zhang2026sla2} which adds learnable routing and quantization-aware training. SALAD \cite{fang2026salad} extends this paradigm to video DiTs, and Sana \cite{xie2024sana} explores pure linear diffusion transformers for efficient image synthesis.

\textbf{Low-Rank and Hardware Optimizations.}
Low-rank approximation is a well-established efficiency technique in LLMs. SLTrain \cite{han2024sltrain} and CR-Net \cite{kong2025crnet} demonstrate the effectiveness of sparse-plus-low-rank approaches for LLM pre-training. However, these methods have not been widely adapted to diffusion transformers, and the "low-rank dilemma" of linear attention \cite{fan2024breaking} remains a key challenge.
In parallel, hardware-optimized attention implementations such as FlashAttention-2 \cite{dao2023flashattention2} and SageAttention \cite{zhang2025sageattention} provide speedups by optimizing memory access and numerical precision.

\section{Preliminary}
\label{sec:preliminary}

\textbf{Notation and Attention Paradigms.} Let $X \in \R^{L \times d_{\text{model}}}$ denote a video sequence flattened from a spatiotemporal grid $p = (t, x, y) \in [T] \times [H] \times [W]$. For each attention head, the input is projected into $Q, K, V \in \R^{L \times d_h}$. Token-specific representations at position $p$ are denoted as $q_p, k_p, v_p \in \R^{d_h}$.

Standard attention computes $O = \softmax(QK^\top / \sqrt{d_h})V$, incurring an intractable $\mathcal{O}(L^2)$ cost for video generation. \textit{Block-sparse attention} circumvents this by restricting queries to a localized neighborhood $\mathcal{N}(p)$:
\begin{equation}
    O_p = \sum_{q \in \mathcal{N}(p)} \frac{\exp\left( q_p^\top k_q / \sqrt{d_h} \right)}{\sum_{r \in \mathcal{N}(p)} \exp\left( q_p^\top k_r / \sqrt{d_h} \right)} v_q
\end{equation}

\textbf{3D RoPE and Shift-Invariance.} 3D RoPE imparts geometric priors via orthogonal rotation. The head dimension is partitioned across axes ($d_h = d_t + d_x + d_y$), with rotational frequencies decaying exponentially:
\begin{equation}
    \theta_m^k = 10000^{-2(m-1)/d_k}, \quad k \in \{t, x, y\}, \quad m \in [1, d_k/2]
\end{equation}
The composite rotation matrix $R(p) \in \R^{d_h \times d_h}$ is block-diagonal. Crucially, the pre-softmax logit preserves shift-invariance, elegantly modeling distance decay:
\begin{equation}
    s_{p,q} = \frac{1}{\sqrt{d_h}} \inner{R(p)q_p}{R(q)k_q} = \frac{1}{\sqrt{d_h}} \inner{q_p}{R(q-p)k_q}
\end{equation}

\textbf{Linear Attention and The RoPE Dilemma.} Linear attention achieves $\mathcal{O}(L)$ complexity by linearizing the exponential kernel with a feature map $\phi(\cdot): \R^{d_h} \to \R^r$, exploiting matrix associativity:
\begin{equation}
    O_p = \phi(q_p)^\top \underbrace{\sum_{q=1}^L \phi(k_q) v_q^\top}_{\text{Global Context } C}
\end{equation}

\textbf{The Dilemma.}\label{The Dilemma of linear attention} The global aggregation of context $C$ neutralizes the explicit relative offset $(p-q)$. Because the \textbf{standard positive feature maps required for matrix associativity cannot preserve the exact orthogonal rotation} ($\phi(R(p)q_p)^\top \phi(R(q)k_q) \neq \phi(q_p)^\top \phi(k_q) \cdot f(p-q)$), the exact RoPE structure cannot be factored out of the summation. This incompatibility forces a severe trade-off between $\mathcal{O}(L)$ efficiency and critical relative positional cues, dictating our absolute-to-relative recovery mechanism in Section \ref{sec:method}.

\section{Method}
\label{sec:method}
\vspace{-0.15cm}
Our proposed \textbf{RoPeSLR} is designed to exploit the sparse-low-rank nature of attention maps in Diffusion Transformers (DiT). 

\subsection{Theoretical Foundation: The Sparse-Low-Rank Theorem}
\label{sec:theoretical_foundation}

Our architecture is motivated by a structural analysis demonstrating that the post-softmax attention map $A$ in 3D RoPE-equipped DiTs intrinsically admits a sparse-low-rank decomposition. We formalize this in the following Theorem (Main Theorem in \ref{theorem: main}, with full proof in Appendix \ref{sec:theory}).

\begin{limbox}
\begin{theorem}[Informal]
\label{thm:main_text_informal} 
    Under Assumptions \ref{assumption 1} and \ref{assumption 2} (empirically validated bounded projection weights and the exponential spectral concentration of 3D RoPE, justified in \textbf{Remark \ref{remark 1}}), for any sequence length $L$, choosing the energy threshold $\tau$ and error tolerance $\mathcal{E}$ as:
\[
\tau = \frac{c}{ L^{\frac{1}{2}}}, \quad \mathcal{E} = \frac{\tau}{2} = \frac{c}{ 2L^{\frac{1}{2}}}
\]
where $c > 0$ is a constant independent of $L$, the post-softmax attention matrix $A$ inherently decomposes into a high-energy sparse target ($A_{\text{sparse}}$) and a smooth background target ($A_{\text{bg}}$). Consequently, there exists a sparse-low-rank reconstructed matrix $\hat{A}_{\text{final}} = \tilde{A}_{\text{sparse}} + \hat{A}_{\text{lowrank}}$ satisfying the following properties:
\begin{enumerate}
 \item \textbf{Sub-quadratic Sparsity:} The residual sparse branch $\tilde{A}_{\text{sparse}}$, which acts as an exact error compensator on the spike set ($\tilde{A}_{\text{sparse}}(p,q) = A(p,q) - \hat{A}_{\text{lowrank}}(p,q)$ for $(p,q) \in \Omega_\tau$), has total non-zero entries bounded by $\text{NNZ}(\tilde{A}_{\text{sparse}}) = \mathcal{O}(L^{\frac{3}{2}})$.
    \item \textbf{Sub-linear Rank:} The globally unmasked low-rank branch $\hat{A}_{\text{lowrank}}$, which acts as a dense approximator for the background target $A_{\text{bg}}$, requires a bottleneck rank of $R = \mathcal{O}\left( d_h \cdot \log L \right)$.
    \item \textbf{Asymptotic Error Bound:} With high probability $1 - \delta_{\text{fail}}$, the global reconstruction error is uniformly bounded by $\max_{p,q} |\hat{A}_{\text{final}}(p,q) - A(p,q)| = \mathcal{O}\left( L^{-\frac{1}{2}} \right)$.
\end{enumerate}
\end{theorem}
\end{limbox}

\textbf{Mathematical Anatomy of 3D RoPE.}
By exploiting orthogonal shift-invariance ($R(p)^\top R(q) = R(q-p)$), the pre-softmax logit admits an exact Fourier series expansion over the relative spatiotemporal offset $\Delta = p - q$. (See lemma \ref{lemma 3D RoPE Induces Exact Fourier Series})
\begin{equation}
\label{eq:rope_fourier_expansion}
s_{p,q} = \sum_{k \in \{t,x,y\}} \sum_{m=1}^{d_k/2} \left( a_{p,m}^{k} \cos(\theta_{m}^{k} \Delta_k) + b_{p,m}^{k} \sin(\theta_{m}^{k} \Delta_k) \right)
\end{equation}
Under this explicit Fourier basis, the background interaction coefficients of pre-trained DiTs exhibit a rapid exponential energy decay. As formally justified in \textbf{Remark \ref{remark 1}} and validated in Appendix \ref{sec:empirical_validation}, this spectral concentration serves as a structural prior that fundamentally decouples the attention manifold into high-frequency semantic spikes and an $\mathcal{O}(d_h \log L)$ low-frequency continuum.

\begin{figure}[htbp]
\centering
\includegraphics[width=0.7\textwidth]{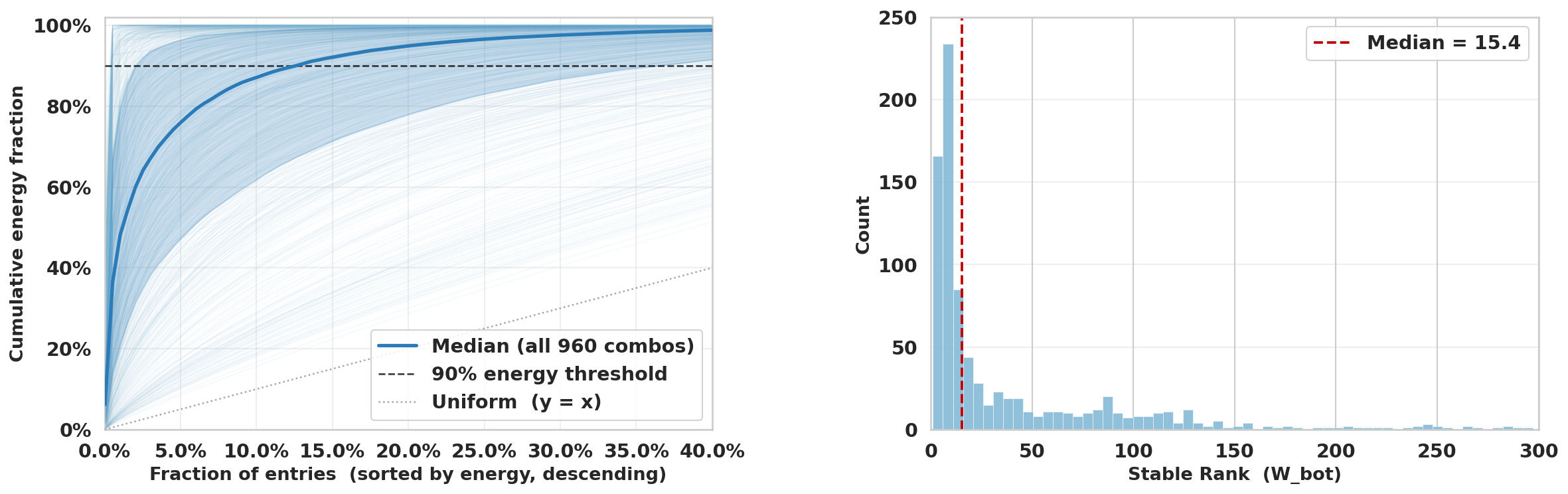}
\caption{
Analysis of HunyuanVideo post-attention matrices across 960 diverse layer-head combinations, verifying their inherent sparse and low-rank structure under a 90\% energy threshold.
}
\label{fig:attn_decomp_hunyuan}
\end{figure}

\textbf{Sparse-Low-Rank Structure of DiTs and LLMs.}
Our framework formally establishes why the sparse-low-rank structure is fundamentally more pronounced in DiTs than in LLMs. This superiority is rooted in two mathematical properties. First, the 3D axis-splitting of RoPE in DiTs induces a \textbf{cubic acceleration in spectral decay} compared to the 1D RoPE standard in LLMs, dictating a drastically lower rank requirement for the background continuum. Second, the inherent spatiotemporal locality of video heavily concentrates semantic spikes, yielding a \textbf{significantly tighter sparsity bound} by avoiding the dense, long-range token retrieval typical of LLMs' causal attention. We defer the formal proofs, including the rank inequalities and sparsity bounds, to Appendix \ref{sec:theory}, Remark \ref{remark: LLM and DIT}.

\begin{limbox_blue}
    \textbf{Empirical Validation.} Analysis of pre-trained DiT attention maps confirms the widespread existence of the sparse-low-rank structure across various layers and denoising steps (Figure \ref{fig:attn_decomp_hunyuan}). As sequence length $L$ scales, the stable rank of the background matrix maintains sublinear while its sparsity increases (Figure \ref{fig:rank_scaling}), directly corroborating the results in Theorem \ref{thm:main_text_informal}. Furthermore, the comparison between Figure \ref{fig:attn_decomp_hunyuan} and Figure \ref{fig:attn_decomp_qwen} validates that this structural decomposition is \textbf{significantly more pronounced in DiTs than in LLMs} (as theoretically justified in Remark \ref{remark: LLM and DIT}).
\end{limbox_blue}
\begin{figure}[htbp]
  \centering
  \includegraphics[width=0.9\textwidth]{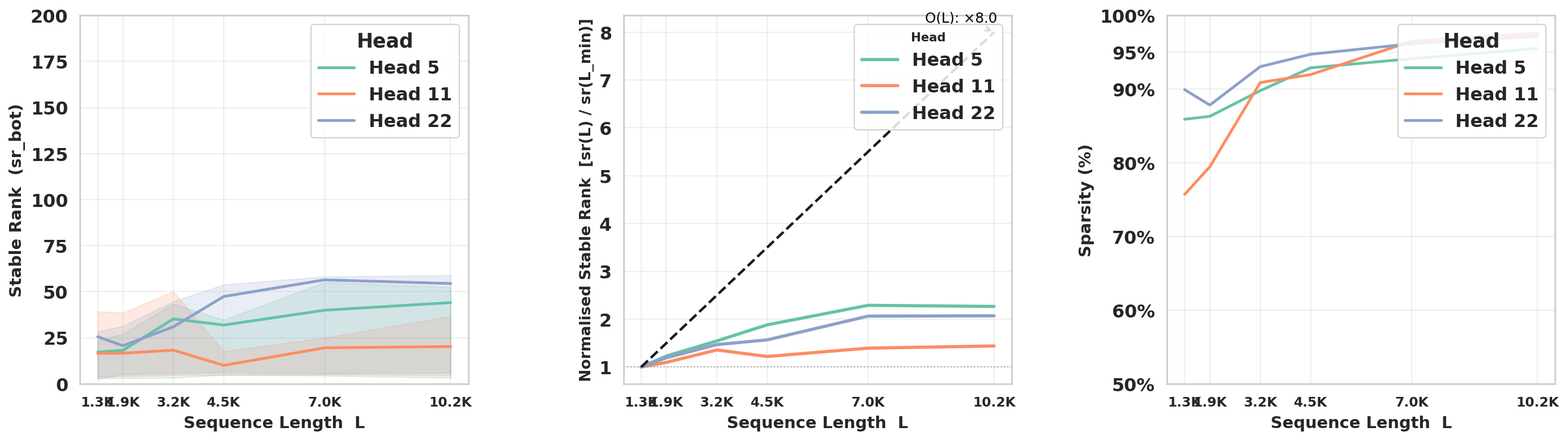} 
  \caption{HunyuanVideo post-softmax matrices are decomposed using a 90\% per-row energy threshold. Across sequence lengths from 1.3K to 10.2K, the low-energy residual maintains bounded stable rank with sublinear growth, aligning with the theoretical bounds established in Theorem \ref{theorem: main}.}
  \label{fig:rank_scaling}
\end{figure}

\subsection{Positional Context Recovery via Absolute PE Injection}
\label{sec:pos_info}

Standard linear attention neutralizes the relative positional offsets essential to 3D RoPE, fundamentally destroying distance awareness (Section \ref{sec:preliminary}). To recover this without reintroducing the $\mathcal{O}(L^2)$ overhead, we leverage our theoretical findings in \textbf{Lemma \ref{lemma 3D RoPE Induces Exact Fourier Series}}. 

Because the 3D RoPE pre-softmax logit expands as a Fourier series over the relative offset $\Delta = p - q$, trigonometric identities allow us to decouple pairwise interactions into products of absolute spatial bases, which allows the global background context to be factorized into query-independent semantic descriptors and query-dependent positional frequencies. This transformation enables a transition from $\mathcal{O}(L^2)$ cross-token routing to an efficient $\mathcal{O}(L)$ point-wise decoding paradigm. We provide the algebraic insight and the structural justification for this transition in \textbf{Remark \ref{remark:token_wise_derivation}}.

We instantiate this decoding via \textbf{Implicit Neural Representations (INRs)} \cite{sitzmann2020implicit}. While standard MLPs suffer from \textit{spectral bias} \cite{rahaman2019spectral} and struggle to capture high-frequency spatial variations, projecting coordinates into periodic bases empowers them to learn high-frequency, translation-invariant mappings \cite{tancik2020fourier}. Accordingly, we inject a learnable 3D Absolute Positional Embedding ($\text{PE}_{\text{3D}}$)—constructed from the trigonometric frequencies ($\theta_m^k$) of our RoPE expansion—into the input $X \in \mathbb{R}^{L \times d_{\text{model}}}$ before routing:
\begin{equation}
    \hat{X} = X + \alpha \odot \text{PE}_{\text{3D}}
    \label{eq:3D PE+X}
\end{equation}
where $\alpha \in \mathbb{R}^{d_{\text{model}}}$ is a learnable gate. Furnished with these absolute geometric bases, our low-rank compensator is equipped to synthesize the long-range relative distance attenuation that linear baselines destroy (\textbf{Remark \ref{remark:token_wise_derivation}}).

\begin{limbox_blue}
    \textbf{Empirical Validation.} Macroscopically, RoPeSLR accurately \textbf{reproduces the exact periodic distance-decay} of full attention (Figure \ref{fig:cosine_similarity}). Mechanistically, Gram spectral analysis (\textbf{Appendix \ref{subsec:svd_mechanistic_analysis}}) proves RoPeSLR successfully decodes absolute coordinates into pure geometric standing waves, which corroborates the insights in \textbf{Remarks \ref{remark:3D ROPE}} and \textbf{\ref{remark:token_wise_derivation}}.
\end{limbox_blue}
\begin{figure}[htbp]
  \centering
  \includegraphics[width=\textwidth]{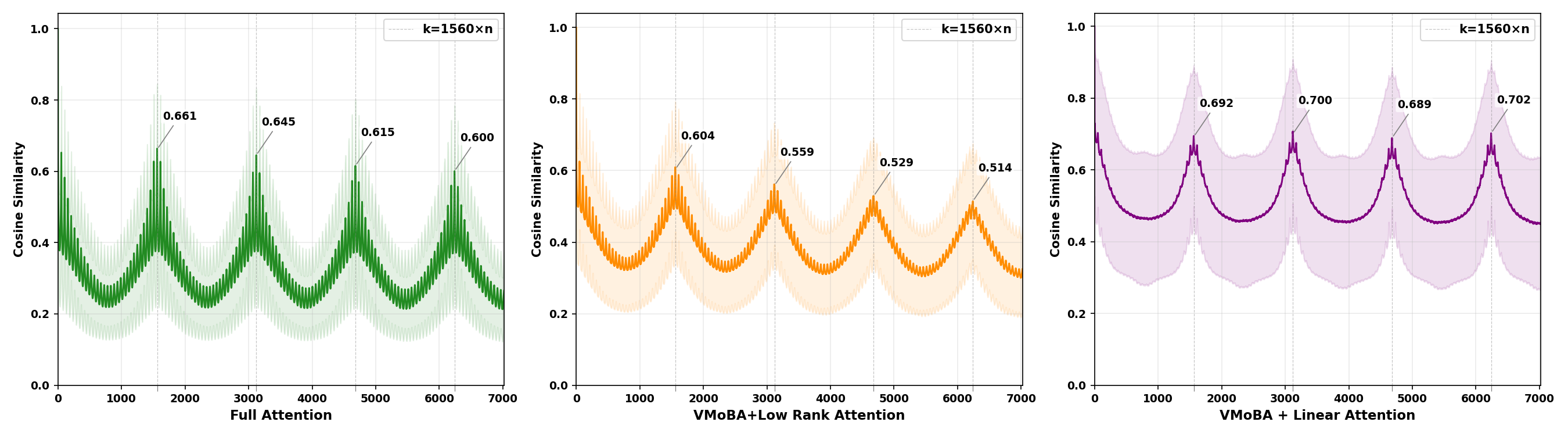}
  \vspace{-0.2cm}
  \caption{
    Statistical analysis of cosine similarity between outputs from three distinct attention variants revealing empirically observed periodic patterns across different token distances.
  }
  \label{fig:cosine_similarity}
\end{figure}

\subsection{Head-wise Sparse-LowRank Decomposition}
\label{sec:decomposition}

To accommodate \textit{head heterogeneity}—the phenomenon where distinct attention heads naturally specialize in varying semantic scales—we eschew monolithic context modeling. Instead, we propose a decoupled architecture: the sparse branch dynamically allocates active token budgets via a continuous energy threshold, while the low-rank compensator employs independent head-wise parameterizations to learn head-specific semantic mappings under a fixed structural rank prior.

\textbf{Energy-Driven Adaptive Sparse Branch.} 
To explicitly capture the non-linear, high-energy ``spikes'' that resist low-rank compression (Theorem \ref{thm:main_text_informal}), we instantiate this branch using the VMoBA framework \cite{wu2025vmoba} governed by a fixed energy threshold $\tau$. We must clarify that we adopt it as a \textbf{practical, hardware-efficient proxy} for our sparse branch. It complements our framework by:
\begin{itemize}[leftmargin=*, itemsep=0em]
    \item \textbf{Empirical Alignment via Energy Thresholding.} Rather than enforcing a rigid, predefined sparsity ratio, the fixed threshold $\tau$ dynamically determines the active block count. This threshold-driven routing inherently yields an adaptive sparsity pattern.
    \item \textbf{Head-wise Adaptivity.} It allocates sparsity budgets independently per head, aligning with the heterogeneous spatiotemporal focus of our decoupled architecture.
    \item \textbf{Hardware Synergy.} It enables high-throughput parallel acceleration on GPUs, ensuring that our extremely high sparsity settings can be effectively exploited for computational efficiency.
\end{itemize}

\textbf{Head-wise Low-Rank Context Compensator.} 
Explicit cross-token aggregation poses a fundamental dilemma: computing exact relative distances incurs an intractable $\mathcal{O}(L^2)$ pairwise routing cost, whereas $\mathcal{O}(L)$ linear approximations neutralize 3D RoPE priors. We resolve this by reframing global context computation as a \textit{conditioned implicit function approximation}. Since Theorem \ref{theorem: main} establishes that the background target $A_{\text{bg}}$ is exceptionally low-rank, the exact context aggregation $(A_{\text{bg}}V)$ intrinsically forms a smooth, low-frequency spatial manifold \cite{wang2020linformer}. 

Rather than explicitly computing sequential aggregation, we introduce a token-wise \textbf{learned non-linear compensator} to act as an amortized proxy. This design exploits the fact that in deep DiT layers, each token $X_p$ has already integrated macroscopic semantic states from preceding layers and global modulations (e.g., AdaLN). Exploiting this dense state encapsulation, our compensator dynamically decodes the smooth global background directly from the modulated token $X_p$ conditioned on its absolute spatial coordinates, bypassing the routing bottleneck entirely (\textbf{see Remark \ref{remark:token_wise_derivation}} for detailed insight). To accommodate head heterogeneity, we instantiate this compensator \textit{independently for each attention head $h$} via a coordinate-guided MLP:
\begin{equation}
    O_{\text{lowrank}}^{(h)} = \sigma\!\left(\sigma(\hat{X} W_A^{(h)}) W_B^{(h)}\right)
    \label{eq:low rank+sigmoid}
\end{equation}
where $\sigma(\cdot)$ is the sigmoid activation, $W_A^{(h)} \in \mathbb{R}^{d_h \times r}$ and $W_B^{(h)} \in \mathbb{R}^{r \times d_h}$ are head-specific projection weights, and $d_h$ is the head dimension. 
Since adaptive sparsity yields greater benefits than adaptive rank, we employ a uniform rank $r$ for implementation convenience, relying on the head-wise parameterization to capture diverse semantics. This decoupled design solves the RoPE dilemma by point-wise decoding the global context, preserving robust non-linear expressivity without the $\mathcal{O}(L^2)$ routing overhead. \textbf{As mechanistically proven via Gram spectral analysis in Appendix \ref{subsec:svd_mechanistic_analysis}}, this explicit PE injection prevents representational collapse, allowing the MLP to decode absolute coordinates into pure, low-frequency geometric standing waves without explicit pairwise computation.

\subsection{Feature Alignment and Gated Fusion}
\label{sec:fusion}

Fusing highly localized sparse features with dense global contexts introduces severe numerical imbalances. We resolve this via a distribution-aware fusion protocol.

\textbf{Distribution Alignment.} We apply independent RMSNorm modules to stabilize the variance of both outputs. This projects the sparse spikes and the dense low-rank approximations onto a shared numerical manifold, guaranteeing balanced optimization:
\begin{equation}
    \tilde{O}_{\text{sparse}}^{(h)} = \text{RMSNorm}\left(O_{\text{sparse}}^{(h)}\right), \quad \tilde{O}_{\text{lowrank}}^{(h)} = \text{RMSNorm}\left(O_{\text{lowrank}}^{(h)}\right)
\end{equation}

\textbf{Input-Dependent Routing.} Since our low-rank compensator functions as a feature-conditioned decoder, the fusion mechanism must dynamically adapt to the input. We compute a token-wise scalar gate $g \in (0, 1)$ via a lightweight projection:
\begin{equation}
    g = \sigma(\hat{X} W_g + b_g)
    \label{eq:gate}
\end{equation}
where $W_g \in \mathbb{R}^{d_{\text{model}} \times 1}, b_g\in \mathbb{R}$. We deliberately bottleneck this gate to a single scalar per token. This ensures $g$ acts strictly as a magnitude regulator, preventing the directional distortion of the aligned semantic feature space that a high-dimensional vector gate would inevitably introduce. 

The final representation is synthesized via a gated residual combination:
\begin{equation}
   O^{(h)} = \tilde{O}_{\text{sparse}}^{(h)} + g \odot \tilde{O}_{\text{lowrank}}^{(h)} 
   \label{eq: output combination}
\end{equation}

\subsection{Post-training Strategy}
\label{sec:post_training}

RoPeSLR is integrated into pre-trained DiTs via a lightweight two-stage post-training pipeline.

\textbf{Stage I: Component Alignment.} We freeze the pre-trained backbone and exclusively optimize the newly introduced parameters $\Theta_{\text{new}} = \{W_A^{(h)}, W_B^{(h)}, \alpha, W_g, b_g\}$. To warm-start the structural decomposition without distorting the pre-trained latent manifold, we minimize the Mean Squared Error (MSE) between our sparse-low-rank output $O$ and the exact full attention output $O_{\text{full}}$.
\begin{equation}
    \mathcal{L}_{\text{align}} = \frac{1}{L \cdot d_{\text{model}}} \| O - O_{\text{full}} \|_F^2
    \label{eq:Stage I loss}
\end{equation}

\textbf{Stage II: Full Fine-tuning.} After alignment, we unfreeze the full backbone and jointly optimize all parameters with the task-specific diffusion objective $L_{\text{task}}$. This stage uses the full model’s representation power to recover end-to-end generation quality while preserving RoPeSLR’s efficiency.

\textbf{Output Alignment Analysis.}
The Stage-I alignment training loss curves (Figure \ref{fig:training_loss}) demonstrate that RoPeSLR achieves faster convergence and a significantly lower loss floor compared to linear baselines. This validates the architectural advantage of our sparse-low-rank design in \textbf{capturing precise relative interactions}, ensuring a robust initialization. Additionally, comprehensive ablation studies on the training schedule (Appendix \ref{sec:Ablation Study on Training Schedule}) confirm that our \textbf{two-stage pipeline is robust}, seamlessly converging to full fine-tuning regardless of the specific Stage-I duration.

\begin{figure}[!htbp]
\centering
\includegraphics[width=0.5\textwidth]{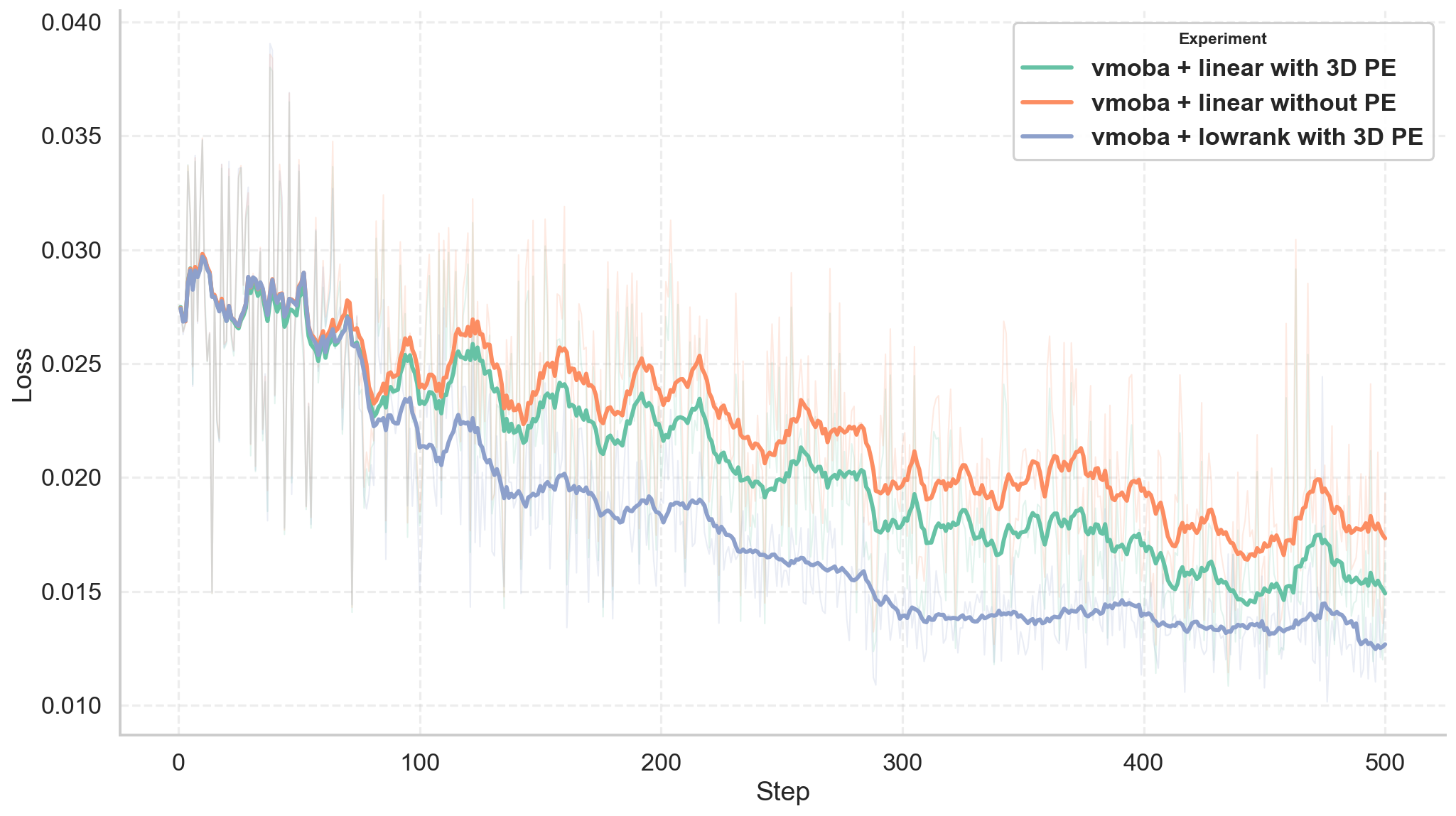} 
\caption{Stage I training loss curves over 500 steps for three configurations: VMoBA with linear attention (no position embedding), VMoBA with linear attention (with 3D learnable position embedding), and our proposed VMoBA + low-rank compensator (with 3D learnable position embedding).
}
\label{fig:training_loss}
\end{figure}

\subsection{Engineering Approaches}
We adopt the official implementation of VMoBA for the sparse attention module. It leverages the kernel to pack tokens into \textbf{FlashAttention}'s varlen layout and \textbf{merges chunk-wise outputs} via log-sum-exp to avoid padding overhead. The low-rank compensator fuses computations with a \textbf{customized Triton kernel} for inference, and uses \textbf{torch.compile for operator fusion} during training. Lazy parameter initialization ensures compatibility with \textbf{gradient checkpointing}.

\section{Experiments}

\subsection{Experimental Settings}
\textbf{Model and Datasets.}
We evaluate RoPeSLR on multiple video generation models (Wan2.1-1.3B, Wan2.1-14B \cite{wan2025wanopenadvancedlargescale}, and Hunyuan13B\cite{kong2024hunyuanvideo}) to validate its cross-scale and cross-architecture generalizability. For fine-tuning, we construct model-specific curated video-text subsets from OpenSora\cite{zheng2024opensora} and OpenVid-1M\cite{nan2024openvid}: Wan2.1-1.3B uses an OpenSora subset, while HunyuanVideo-13B and Wan2.1-14B use OpenVid-1M subsets. All subsets contain high-quality clips with diverse motions and scenes, preprocessed to match the corresponding base model’s resolution and frame configuration.\footnote{The sparsity ratios of baseline methods follow the default settings of their official native implementations. Since some sparse methods cannot support precise sparsity control, all reported sparsity values in this work are within a deviation of 1\%.}

\textbf{Baselines.} We further compare RoPeSLR with latest sparse and sparse-linear attention approaches,all evaluated using official implementations. For baselines incompatible with HunyuanVideo, we migrate them following their original design:
\begin{itemize}[noitemsep, topsep=0pt, parsep=0pt]
    \item \textbf{SLA\cite{zhang2025sla}}: A fine-tunable sparse-linear attention framework that fuses sparse and linear attention to accelerate DiTs while preserving generation quality. 
    \item \textbf{VSA\cite{zhang2025vsa}}: A trainable sparse attention paradigm designed to accelerate video diffusion Transformers by learning adaptive sparse patterns over spatio-temporal tokens.
    \item \textbf{SVG2\cite{yang2025sparse}}: The improved version of Sparse VideoGen\cite{xi2025sparse}, which accelerates video generation with spatial-temporal sparse attention via semantic-aware permutation.

\end{itemize}

\textbf{Metrics.} For video quality, we evaluate on the full VBench\cite{huang2023vbenchcomprehensivebenchmarksuite} suite and report key evaluation dimensions: Subject Consistency (SC), Motion Smoothness (MS), Background Consistency (BC), Imaging Quality (IQ) and Aesthetic Quality (AQ). We adopt FLOPs to quantify the computational cost of the attention module. For fair comparison, we report only the FLOPs associated with attention computation per denoising step, excluding text tokens, QKV projections and MLP blocks.

\subsection{Experimental Results Analysis}
\vspace{-0.25cm}

\begin{table*}[htbp]
  \centering
  \fontsize{7}{8}\selectfont  
  \setlength{\tabcolsep}{2.1pt} 
  \caption{Quantitative results on Wan2.1-1.3B and Hunyuan13B. We evaluate Wan2.1-1.3B on 81-frame $480 \times 832$ videos and Hunyuan13B on 129-frame $720 \times 1280$ videos, all with NVIDIA A100 80G GPUs.}
  \label{tab:main_benchmark}
  \renewcommand{\arraystretch}{1.0}
  \resizebox{\textwidth}{!}{
  \begin{tabular}{llccccccc}
    \toprule
    \multirow{2}{*}{\textbf{Model}} &
    \multirow{2}{*}{\textbf{Method}} &
    \multicolumn{5}{c}{\textbf{VBench Quality}} &
    \multicolumn{2}{c}{\textbf{Efficiency}} \\
    \cmidrule(lr){3-7} \cmidrule(lr){8-9}
    & &
    \textbf{SC↑} &
    \textbf{MS↑} &
    \textbf{BC↑} &
    \textbf{IQ↑} &
    \textbf{AQ↑} &
    \textbf{FLOPs↓} &
    \textbf{Sparsity↑} \\
    \midrule

    \multirow{5}{*}{Wan2.1-1.3B}
    & Full   & 0.9426 & 0.9553 & 0.9474 &  \textbf{0.6506} & \textbf{0.6117} & 197.82 TFLOPs & 0\%  \\
    & SVG2   & 0.9133 & 0.9579 & 0.9287 & 0.6369 & 0.5625 & 54.38 TFLOPs & 72\% \\
    & VSA    & 0.9042 & 0.9862 & 0.9425 & 0.6389 & 0.5542 & \textbf{19.98 TFLOPs} & 89\%  \\
    & SLA    & 0.9377 & 0.9776 & 0.9485 & 0.6258 & 0.5975 & 20.94 TFLOPs & 90\%  \\
    & \cellcolor{blue!10}\textbf{RoPeSLR (Ours)}
    & \cellcolor{blue!10}\textbf{0.9510}
    & \cellcolor{blue!10}\textbf{0.9898}
    & \cellcolor{blue!10}\textbf{0.9487}
    & \cellcolor{blue!10}{0.6440}
    & \cellcolor{blue!10}{0.6071}
    & \cellcolor{blue!10}{20.17 TFLOPs}
    & \cellcolor{blue!10}\textbf{90\%} \\
    \midrule

    \multirow{5}{*}{HunyuanVideo-13B}
    & Full & \textbf{0.9717} & \textbf{0.9957} & \textbf{0.9744} & \textbf{0.6771} & \textbf{0.6232} & 10.41 PFLOPs & 0\% \\
    & SVG2 & 0.9411 & 0.9653 & 0.9537 & 0.6328 & 0.6032 & 2.66 PFLOPs & 75\% \\
    & VSA  & 0.9442 & 0.9762 & 0.9625 & 0.6088 & 0.6142 & \textbf{1.04 PFLOPs} & 90\% \\
    & SLA  & 0.9592 & 0.9896 & 0.9665 & 0.6216 & 0.6043 & 1.06 PFLOPs & 90\% \\
    & \cellcolor{blue!10}\textbf{RoPeSLR (Ours)}
    & \cellcolor{blue!10} 0.9697
    & \cellcolor{blue!10} 0.9930
    & \cellcolor{blue!10} 0.9706
    & \cellcolor{blue!10} 0.6597
    & \cellcolor{blue!10} 0.6179
    & \cellcolor{blue!10} 1.05 PFLOPs
    & \cellcolor{blue!10}\textbf{90\%} \\
    \bottomrule
    \end{tabular}
  }
\end{table*}

\textbf{Quality Evaluation.}
As shown in Table~\ref{tab:main_benchmark}, even at the extreme sparsity of 90\%, RoPeSLR substantially reduces computational overhead compared with dense full attention. Averaged across all VBench dimensions, our method incurs a performance drop of less than 1.3\% against full attention, with negligible degradation on individual metrics and well-preserved generation fidelity. This strongly validates its advantages in the \textbf{quality-efficiency trade-off}. Quantitative results for Wan2.1-14B are presented in Table \ref{tab:wan14b_benchmark}.

\textbf{Latency Characteristics and Sparse Branch Flexibility.}
For compatibility, we adopt VMoBA as the default sparse branch, though its implementation is not fully optimized, leading to lower end-to-end latency reduction than the theoretical FLOPs reduction. Notably, RoPeSLR is sparse-backbone-agnostic and supports seamless substitution of VMoBA with other advanced sparse attention designs. Detailed latency and efficiency evaluations are provided in Figure \ref{fig:lowrank_efficiency}, Table \ref{tab:long_sequence_latency}, and Section \ref{sec:complexity_analysis}..

\textbf{Sparsity-Aligned Comparison.}
Further sparsity-aligned evaluation on SVG2 (see Table~\ref{tab:svg2_vsa_sparsity_align}) demonstrates that baseline perform significantly worse than RoPeSLR under identical sparsity constraints.

\textbf{Supplementary Analyses.}
We provide comprehensive supplementary validations: 
(i) visualize the adaptive token-wise gate distributions in Appendix \ref{subsec:gate_vis}; 
(ii) verify that diffusion Transformers (DiTs) exhibit a more prominent sparse-low-rank structural property than LLMs in Appendix \ref{subsec:llm_attn_char}; 
(iii) list key training hyperparameters in Appendix \ref{subsec:hyperparams};
(iv) validate compatibility with step-distilled models for compounded acceleration in Appendix \ref{app:distillation}.

\subsection{Ablation Study}
\textbf{Experimental Settings.}
All ablation experiments are conducted on the Wan2.1-1.3B model with identical settings to the main experiments. Core ablations are consolidated in Table~\ref{tab:ablation_all}.

\textbf{Low-Rank Compensator Validation.}
Lines 1–4 compare different context modeling strategies, alongside an equal-parameter MLP baseline\footnote{The equal-parameter MLP baseline is parameter-matched with our low-rank branch, trained with random initialization and without Stage I pre-alignment or positional encoding (PE). This controls confounding variables to isolate the performance gains from our theoretically derived structural design.}. This setup distinguishes performance gains brought by our structural design from simple parameter expansion or non-linear enhancement. Results demonstrate that our low-rank compensator outperforms linear attention and the equal-parameter MLP, verifying the advantages of our theoretically derived sparse low-rank design.

\textbf{Positional Encoding Necessity and Design.}
Lines 5–7 evaluate the influence of positional priors on spatiotemporal context decoding. Incorporating 3D positional encoding steadily improves generation performance, confirming its indispensable role in global context aggregation. 

\textbf{Additional Ablation Studies.}
Supplementary analyses in the appendix include sparsity optimization (see Table~\ref{tab:sparsity_ablation}), which identifies 90\% sparsity as the optimal balance point, and activation function ablation (see Table~\ref{tab:activation_functions_low_rank}), which validates Sigmoid as the most suitable choice for the low-rank compensator. Furthermore, we empirically ablate the low-rank dimension $r$ (See Appendix \ref{sec:ablation_low_rank_dimension}).

\begin{table}[!htbp]
\centering
\fontsize{7}{8}\selectfont    
\setlength{\tabcolsep}{2.5pt}   
\renewcommand{\arraystretch}{1.2} 
\caption{Component Ablation Studies on RoPeSLR. }
\label{tab:ablation_all}
\begin{tabular}{lccccc}
\toprule
Method & SC & BC & MS & AQ & IQ \\
\midrule
\multicolumn{6}{l}{\textit{Context Modeling and Structural Baselines}} \\
1. VMoBA only & 0.8496 & 0.8909 & 0.9580 & 0.4740 & 0.5294 \\
2. VMoBA + LinearAttention (without PE) & 0.9310 & 0.9267 & 0.9776 & 0.5621 & 0.6104 \\
3. VMoBA + Equal-Param MLP (without PE) & 0.9002 & 0.9149 & 0.9789 & 0.4835 & 0.5497 \\
4. VMoBA + Low-Rank Compensator (without PE) & \textbf{0.9452} & \textbf{0.9336} & \textbf{0.9882} & \textbf{0.5756} & \textbf{0.6332} \\
\midrule
\multicolumn{6}{l}{\textit{Positional Encoding Design Comparison}} \\
5. VMoBA + Low-Rank Compensator (with 1D PE) & 0.9454 & 0.9342 & 0.9886 & 0.5865 & 0.6338 \\
6. VMoBA + LinearAttention (with 3D PE) & 0.9476 & 0.9421 & 0.9869 & 0.5808 & 0.6380 \\
7. VMoBA + Low-Rank Compensator (with 3D PE) & \textbf{0.9510} & \textbf{0.9487} & \textbf{0.9898} & \textbf{0.6071} & \textbf{0.6440} \\
\bottomrule
\end{tabular}
\end{table}

\section{Conclusion and Limitation}
\label{sec:Conclusion and Limitation}

\textbf{Conclusion.} We introduce RoPeSLR, a theoretically grounded attention framework that resolves the $\mathcal{O}(L^2)$ bottleneck in DiTs. We prove that 3D RoPE decouples the attention manifold into a sparse-low-rank structure. Extensive evaluations confirm that RoPeSLR accelerates inference at extreme sparsity while strictly preserving high-fidelity generation quality.

\noindent\textbf{Future Work.} While our default sparse branch (VMoBA) reduces theoretical FLOPs, its sub-optimal engineering limits end-to-end latency gains. Because RoPeSLR is backbone-agnostic, future work will integrate hardware-optimized sparse kernels to fully translate our theoretical efficiency into real-world speedups.
{
\small
\bibliography{reference}
\bibliographystyle{abbrv}
}

\appendix
\newpage
\begin{center}
\LARGE
    \textbf{Appendix}
\end{center}
\vspace{5mm}
\tableofcontents

\vspace{5mm}

\newpage

\newpage

\section{Theory}
\label{sec:theory}
\subsection{Proof Overview}
We provide an analytical existence proof demonstrating that the ground-truth post-softmax attention matrix in 3D RoPE-equipped and pre-trained video diffusion transformers s(DiTs) intrinsically admits a clean decomposition into a highly sparse branch and a provably low-rank compensator. 
Crucially, our theoretical framework is designed to prove the \textbf{structural existence} of this mathematical decomposition, revealing the geometric properties of the attention matrix itself, \textbf{independent of the forward-pass computational algorithms.}

\begin{enumerate}[leftmargin=*, itemsep=0.2em]
    \item We first define an energy threshold $\tau$ to split the full attention matrix $A$ into two disjoint components: a high-energy spike set $\OmegaSpike$ (captured by the sparse branch) and a low-energy smooth background set $\OmegaBg$ (approximated by the low-rank compensator). We prove the sparse branch has a deterministic sparsity guarantee: at most $\lfloor 1/\tau \rfloor$ non-zero entries per row.
    \item Leveraging the mathematical structure of 3D RoPE, we prove the pre-softmax QK logit matrix can be exactly expanded as a Fourier series over spatiotemporal frequencies, where each frequency term corresponds to a matrix of rank at most 2.
    \item Under a empirically validated spectral concentration assumption, high-frequency interaction coefficients in the background set decay exponentially, allowing us to truncate the Fourier series to get a low-rank approximation of the pre-softmax background matrix with controllable error.
    \item We use positive random features (FAVOR+)\cite{choromanski2021rethinking} to linearize the softmax exponential kernel while preserving the low-rank structure, and perform error propagation to bound the final approximation error of the background branch.
    \item Our final main theorem establishes the existence of this decomposition: the sparse branch has guaranteed high sparsity, and the low rank branch has a rank bound that is linear in head dimension $d_h$ and only logarithmic in sequence length $L$, confirming extreme low-rankness for long video sequences.
\end{enumerate}

\subsection{Preliminaries and Notation}
\label{sec:prelim}
We first formalize all symbols and foundational definitions with zero ambiguity, consistent with standard video diffusion transformer literature.

\textbf{Notation Conventions}
\begin{itemize}[leftmargin=*, itemsep=0.3em]
    \item \textbf{Video Tokenization}: A video is tokenized into a sequence of spatiotemporal tokens indexed by $p = (t, x, y)$, where $t$ is the temporal index, $x,y$ are spatial indices. The total sequence length is $L = T \cdot H \cdot W$, with $T$ temporal frames, $H \times W$ spatial resolution per frame.
    \item \textbf{Relative Position}: For tokens $p,q$, the per-axis relative position is $\Delta_k = p_k - q_k$ for $k \in \{t,x,y\}$, where $p_k$ denotes the $k$-th coordinate of token $p$.
    
    \item \textbf{3D Rotary Position Embedding (3D RoPE)}: The head dimension is split as $d_h = d_t + d_x + d_y$, with $d_k$ the dimension allocated to axis $k$. For axis $k$, the $m$-th frequency pair uses the standard exponential schedule:
    \[
    \theta_m^k = 10000^{-2(m-1)/d_k}, \quad m = 1,2,\dots,d_k/2
    \]
    The 2D rotation matrix for axis $k$, frequency $m$, and position $p_k$ is:
    \[
    R_k(p_k, m) = \Rot{\theta_m^k p_k}
    \]
    The full RoPE rotation matrix $R(p) \in \R^{d_h \times d_h}$ is block-diagonal, with each block corresponding to $R_k(p_k, m)$ for each axis and frequency. By construction, $R(p)$ is orthogonal: $R(p)^\top R(p) = I$.
    \item \textbf{Frequency Subspace Projection}: For token $p$, axis $k$, frequency $m$, we define the 2D projections of query/key vectors onto the corresponding frequency subspace:
    \[
    Q_p^{(k,m)} = \begin{pmatrix} q_{p, 2(m-1)+1}^{(k)} \\ q_{p, 2m}^{(k)} \end{pmatrix}, \quad K_q^{(k,m)} = \begin{pmatrix} k_{q, 2(m-1)+1}^{(k)} \\ k_{q, 2m}^{(k)} \end{pmatrix}
    \]
    where $q_p^{(k)} \in \R^{d_k}$ is the part of $q_p$ allocated to axis $k$.
    \item \textbf{Attention Matrices}:

    \textbf{Pre-softmax QK logit matrix}: $s_{p,q} = \inner{R(p)q_p}{R(q)k_q}$ (we omit the standard $1/\sqrt{d_h}$ scaling for clarity; it can be reintroduced as a linear scaling factor without affecting any bounds).
    
    \textbf{Post-softmax attention matrix:} $A_{p,q} = \frac{\exp(s_{p,q})}{Z_p}$, where $Z_p = \sum_{r=1}^L \exp(s_{p,r})$ is the row partition function. 

    \textbf{Spike/Background Sets}: For a given energy threshold $\tau > 0$, we define:
      High-energy spike set: $\OmegaSpike = \{(p,q) \mid A_{p,q} > \tau\}$; 
      Low-energy background set: $\OmegaBg = \{(p,q) \mid A_{p,q} \leq \tau\}$.
    
    \textbf{Global Error Tolerance}: Fixed constant $\mathcal{E} > 0$, the maximum allowable absolute error per element for the low-rank approximation on the background set.
\end{itemize}

\textbf{Foundational Facts.} 
We list standard mathematical facts used throughout the proofs, to eliminate any implicit assumptions:
\begin{itemize}
    \item \textbf{Cauchy-Schwarz Inequality.} For any vectors $u,v \in \R^d$, $\abs{\inner{u}{v}} \leq \norm{u}_2 \norm{v}_2$.
    \item \textbf{Sylvester's Rank Inequality.} For matrices $A \in \R^{m \times n}$, $B \in \R^{n \times p}$, $\rank(AB) \leq \min(\rank(A), \rank(B))$.
    \item \textbf{Union Bound (Boole's Inequality).} For events $E_1, E_2, \dots, E_n$, $\PP\left(\bigcup_{i=1}^n E_i\right) \leq \sum_{i=1}^n \PP(E_i)$.
    \item \textbf{Exponential Function Inequalities.} For all $x \in \R$, $e^x \geq 1+x$. For $0 \leq x \leq 0.5$, $e^x \leq 1+2x$.
    \item \textbf{Invertible Matrix Preserves Rank.} For any matrix $A \in \R^{m \times n}$ and invertible matrix $B \in \R^{n \times n}$, $\rank(AB) = \rank(A)$.
\end{itemize}

\subsection{Energy-Based Sparse-Low-Rank Decomposition}
We first formalize the core decomposition of the attention matrix, which is the foundation of our framework.
\begin{definition}[Energy-Threshold Sparse-Low-Rank Decomposition]
\label{def:1}
For a given energy threshold $\tau \in (0, 1)$, we decompose the full post-softmax attention matrix $A$ into two disjoint, complementary components:
\[
A = A_{\text{sparse}} + A_{\text{bg}}
\]
where the sparse branch isolates high-energy semantic spikes:
\[
A_{\text{sparse}}(p,q) = A(p,q), if {(p,q) \in \OmegaSpike}
\]
and the background branch (low-rank approximation target) contains only the low-energy smooth residual:
\[
A_{\text{bg}}(p,q) = A(p,q) , if{(p,q) \in \OmegaBg} 
\]
\end{definition}

\begin{definition}[Residual Sparse Branch for Reconstruction]
\label{def:2}
Because the background target $A_{\text{bg}}$ contains zeros at the high-energy spike locations $\OmegaSpike$, directly fitting it with a low-rank matrix is highly inefficient. Instead, we introduce a globally dense, unmasked low-rank approximator $\hat{A}_{\text{lowrank}}$ (whose existence is proven in Corollary \ref{corollary: main}) to approximate $A_{\text{bg}}$ only on the background set $\OmegaBg$. 

To correct the inaccurate outputs produced by $\hat{A}_{\text{lowrank}}$ on the spike set $\OmegaSpike$, we define the residual sparse branch $\tilde{A}_{\text{sparse}}$ as an exact error compensator:
\[
\tilde{A}_{\text{sparse}}(p,q) = \begin{cases} A(p,q) - \hat{A}_{\text{lowrank}}(p,q) & \text{if } (p,q) \in \OmegaSpike \\ 0 & \text{otherwise} \end{cases}
\]
The final reconstructed attention matrix is then guaranteed as:
\[
\hat{A}_{\text{final}} = \tilde{A}_{\text{sparse}} + \hat{A}_{\text{lowrank}}
\]
\end{definition}

\subsection{Core Assumptions}
We state all non-trivial assumptions used in this work, all of which are empirically validated in video diffusion transformers.

\begin{assumption}[Bounded Query/Key Projection Weights]
\label{assumption 1}
The spectral norms of the trainable query and key projection matrices are uniformly bounded by a constant independent of sequence length $L$ and head dimension $d_h$:
\[
\|W_q\|_2 \leq W_{\text{max}}, \quad \|W_k\|_2 \leq W_{\text{max}}
\]
where $W_{\text{max}} > 0$ is a fixed constant.
\end{assumption}
\begin{assumption}[Spectral Concentration of Background Interaction Coefficients]
\label{assumption 2}
For the background set of a \textbf{pretrained DiT model} $\OmegaBg = \{(p,q) \mid A_{p,q} \leq \tau\}$, the maximum high-frequency interaction coefficients decay exponentially with frequency index for all three axes:
\[
\sum_{m=m_0}^{d_k/2} \max_{(p,q) \in \OmegaBg} \left( |a_{p,m}^k| + |b_{p,m}^k| \right) \leq C_\omega \cdot \rho_k^{m_0}, \quad \forall m_0 \geq 1
\]
where $\rho_k = 10000^{-\alpha/d_k} < 1$, $C_\omega > 0, \alpha > 0$ are constants independent of sequence length $L$ and head dimension $d_h$. Here, $a_{p,m}^k$ and $b_{p,m}^k$ are the Fourier coefficients from the 3D RoPE expansion (Lemma \ref{lemma 3D RoPE Induces Exact Fourier Series}).
\end{assumption}

\begin{remark}[Theoretical Grounding and Empirical Validity of Assumption \ref{assumption 2}]
\label{remark 1}
We note that Assumption \ref{assumption 2} functions as a grounded structural prior characterizing the \textit{macroscopic} background continuum, rather than a rigid universal law. Its formulation is motivated by two established theoretical properties: 
(1) \textbf{The RoPE Long-Term Decay Prior}: The exponential schedule of RoPE's base frequencies ($\theta_m^k = 10000^{-2(m-1)/d_k}$) intrinsically penalizes high-frequency interactions over extended distances \cite{su2021roformer}. 
(2) \textbf{The Low-Rank Spectrum of Softmax Attention}: The singular value spectrum of the residual background attention naturally exhibits exponential decay \cite{wang2020linformer}.

Crucially, the practical validity of this prior is structurally enforced by our decoupled architecture and extensively corroborated by our empirical analyses. Because well-optimized DiTs naturally decouple sharp semantic spikes from the smooth background, explicitly isolating these high-energy interactions into the sparse branch guarantees that the remaining background set $\OmegaBg$ strictly conforms to this low-frequency profile. We validate this phenomenon in \textbf{Appendix \ref{sec:empirical_validation} (Figure \ref{fig:denoising_evolution})}, which clearly demonstrates the consistent exponential spectral decay of QK interactions across multiple foundational DiTs. Furthermore, our stability analysis (\textbf{Figure \ref{fig:rank_scaling}}) confirms that the background continuum maintains bounded, sub-linear rank growth even as the sequence length scales significantly. Together, these structural and empirical insights perfectly justify the efficient $\mathcal{O}(d_h \log L)$ bottleneck rank bound established in Theorem \ref{thm:main_text_informal}.
\end{remark}

\begin{theorem}[Sparsity Bound and Background Energy Suppression]
\label{theorem 1}
Given the decomposition in Definition 1, the following two results hold deterministically for any $\tau \in (0,1)$:
\begin{itemize}
    \item \textbf{Sparsity Bound}: The sparse branch $A_{\text{sparse}}$ contains at most $\lfloor 1/\tau \rfloor$ non-zero entries per row. The total number of non-zero entries is bounded above by $L \cdot \lfloor 1/\tau \rfloor$.
    
    \textbf{Asymptotic Sparsity Analysis}: For any sequence length $L$, the total number of non-zero entries in the sparse branch is bounded by:
    \[
    \text{NNZ}(A_{\text{sparse}}) = O(L \cdot \left\lfloor \frac{1}{\tau} \right\rfloor)
    \].
    
    \item \textbf{Background Element-wise Bound}: The background branch $A_{\text{bg}}$ is bounded element-wise: $\|A_{\text{bg}}\|_\infty \leq \tau$.
\end{itemize}
\end{theorem}

\begin{proof}
We prove each statement sequentially, with explicit justification for every step.

\paragraph{Proof of Statement 1 (Sparsity Bound)}
Fix an arbitrary row $p$ of the attention matrix. Let $N_p$ denote the number of entries in row $p$ with $A_{p,q} > \tau$. By definition of row-stochasticity, we have:
\[
\sum_{q=1}^L A_{p,q} = 1, \quad \text{and } A_{p,q} \geq 0 \text{ for all } q
\]
We split the sum into entries in the spike set and background set:
\[
\sum_{q=1}^L A_{p,q} = \sum_{q: A_{p,q} > \tau} A_{p,q} + \sum_{q: A_{p,q} \leq \tau} A_{p,q} = 1
\]
Since all terms are non-negative, we can drop the background sum to get a lower bound:
\[
\sum_{q: A_{p,q} > \tau} A_{p,q} \leq 1
\]
For every entry in the spike set, $A_{p,q} > \tau$ by definition. There are $N_p$ such entries, so:
\[
\sum_{q: A_{p,q} > \tau} A_{p,q} > \sum_{q: A_{p,q} > \tau} \tau = N_p \cdot \tau
\]
Combining the two inequalities:
\[
N_p \cdot \tau < 1 \implies N_p < \frac{1}{\tau}
\]
Since $N_p$ is a non-negative integer, the maximum possible value of $N_p$ is the largest integer less than $1/\tau$, which is $\lfloor 1/\tau \rfloor$. Thus:
\[
N_p \leq \left\lfloor \frac{1}{\tau} \right\rfloor \quad \text{for all rows } p
\]
Summing over all $L$ rows, the total number of non-zero entries in $A_{\text{sparse}}$ is bounded above by:
\[
\sum_{p=1}^L N_p \leq L \cdot \left\lfloor \frac{1}{\tau} \right\rfloor
\]
This completes the proof of Statement 1.

\paragraph{Proof of Statement 2 (Background Element-wise Bound)}
By definition of $A_{\text{bg}}$ (Definition \ref{def:1}), for any $(p,q) \in \OmegaBg$, we have $A_{\text{bg}}(p,q) = A_{p,q} \leq \tau$. For $(p,q) \in \OmegaSpike$, $A_{\text{bg}}(p,q) = 0 \leq \tau$. Thus, for all $p,q$, $A_{\text{bg}}(p,q) \leq \tau$. The infinity norm of a matrix is the maximum absolute value of its entries, so:
\[
\|A_{\text{bg}}\|_\infty = \max_{p,q} |A_{\text{bg}}(p,q)| \leq \tau
\]
This completes the proof of Statement 2.

\end{proof}

\subsection{Core Lemmas}

\begin{lemma}[Deterministic Q/K Boundedness]
\label{lemma Q/K Boundedness}
By Assumption \ref{assumption 1}, for any diffusion time step $s$, the query and key vectors are deterministically bounded across all tokens:
\[
\max_{p} \left( \| q_p \|_2, \| k_p \|_2 \right) \leq Q_{\text{max}} = W_{\text{max}} \sqrt{d_h}
\]
where $Q_{\text{max}}$ is a constant independent of sequence length $L$.
\end{lemma}

\begin{proof}

\paragraph{Step 1: Deterministic Norm Bound for LayerNorm Output}
By definition of LayerNorm, for any token feature vector $x_{s,p}$, we have:
\[
\text{LN}(x_{s,p}) = \frac{x_{s,p} - \mu}{\sigma}, \quad \sigma^2 = \frac{1}{d_h}\sum_{i=1}^{d_h} (x_{s,p,i} - \mu)^2 
\]
The squared $\ell_2$ norm of the LayerNorm output is:
\[
\|\text{LN}(x_{s,p})\|_2^2 = \sum_{i=1}^{d_h} \left( \frac{x_{s,p,i} - \mu}{\sigma} \right)^2 = \frac{1}{\sigma^2} \sum_{i=1}^{d_h} (x_{s,p,i} - \mu)^2
\]
By definition of $\sigma^2$, $\sum_{i=1}^{d_h} (x_{s,p,i} - \mu)^2 = d_h \sigma^2$. Substituting this in, we have
\[
\|\text{LN}(x_{s,p})\|_2^2 = \frac{1}{\sigma^2} \cdot d_h \sigma^2 = d_h
\]
Taking the square root of both sides gives the deterministic norm bound:
\[
\|\text{LN}(x_{s,p})\|_2 = \sqrt{d_h}
\]

\paragraph{Step 2: Norm Bound for Query/Key Vectors}
By definition, the query vector is $q_p = W_q \cdot \text{LN}(x_{s,p})$. We use the submultiplicativity property of the spectral norm: for any matrix $W$ and vector $x$, $\|Wx\|_2 \leq \|W\|_2 \|x\|_2$. Applying this property:
\[
\|q_p\|_2 = \|W_q \cdot \text{LN}(x_{s,p})\|_2 \leq \|W_q\|_2 \cdot \|\text{LN}(x_{s,p})\|_2
\]
Substituting the bounds $\|W_q\|_2 \leq W_{\text{max}}$ and $\|\text{LN}(x_{s,p})\|_2 = \sqrt{d_h}$:
\[
\|q_p\|_2 \leq W_{\text{max}} \sqrt{d_h}
\]
The identical derivation applies to the key vector $k_p$, since $\|W_k\|_2 \leq W_{\text{max}}$. Since this bound holds for all tokens $p$, we have:
\[
\max_{p} \left( \| q_p \|_2, \| k_p \|_2 \right) \leq W_{\text{max}} \sqrt{d_h}
\]
This completes the proof.\footnote{In practice, LayerNorm applies a learned affine transformation $\gamma \odot \hat{x} + \beta$. We assume these learned scaling and shifting parameters are absorbed into the bounded query/key projection matrices $W_q$ and $W_k$ without loss of generality, maintaining the rigor of the deterministic norm bound.}
\end{proof}

\begin{lemma}[3D RoPE Induces Exact Fourier Series]
\label{lemma 3D RoPE Induces Exact Fourier Series}
With standard 3D RoPE, the QK inner product can be written exactly as:
\[
s_{p,q} = \sum_{k \in \{t,x,y\}} \sum_{m=1}^{d_k/2} \left( a_{p,m}^k \cos(\theta_m^k \Delta_k) + b_{p,m}^k \sin(\theta_m^k \Delta_k) \right)
\]
where $\Delta_k = p_k - q_k$, and the coefficients are defined as:
\[
a_{p,m}^k = Q_p^{(k,m),1} K_q^{(k,m),1} + Q_p^{(k,m),2} K_q^{(k,m),2}, \quad b_{p,m}^k = Q_p^{(k,m),1} K_q^{(k,m),2} - Q_p^{(k,m),2} K_q^{(k,m),1}
\]
and satisfy the bounds:
\[
|a_{p,m}^k| \leq \| Q_p^{(k,m)} \|_2 \| K_q^{(k,m)} \|_2, \quad |b_{p,m}^k| \leq \| Q_p^{(k,m)} \|_2 \| K_q^{(k,m)} \|_2
\]
\end{lemma}

\begin{proof}
\paragraph{Step 1: Simplify the RoPE QK Inner Product}
By definition, the pre-softmax logit is:
\[
s_{p,q} = \inner{R(p)q_p}{R(q)k_q}
\]
Since $R(p)$ is orthogonal, $R(p)^\top R(p) = I$. We use the standard RoPE identity for relative position encoding:
\[
\inner{R(p)q_p}{R(q)k_q} = \inner{q_p}{R(p)^\top R(q) k_q} = \inner{q_p}{R(q-p) k_q}
\]
This follows directly from the orthogonality of RoPE rotation matrices.

\paragraph{Step 2: Block-Diagonal Expansion of the Rotation Matrix}
The RoPE rotation matrix $R(\Delta)$ (where $\Delta = p-q$) is block-diagonal, with each block corresponding to a 2D rotation matrix for a specific axis and frequency. Thus, the inner product decomposes into a sum over axes and frequencies:
\[
\inner{q_p}{R(\Delta) k_q} = \sum_{k \in \{t,x,y\}} \sum_{m=1}^{d_k/2} \inner{Q_p^{(k,m)}}{R_k(\Delta_k, m) K_q^{(k,m)}}
\]
where $Q_p^{(k,m)}, K_q^{(k,m)}$ are the 2D frequency subspace projections defined in Section 2.1, and $R_k(\Delta_k, m) = \Rot{\theta_m^k}{\Delta_k}$ is the 2D rotation matrix for axis $k$, frequency $m$, relative position $\Delta_k$.

\paragraph{Step 3: Explicit 2D Rotation and Fourier Series Form}
We expand the 2D inner product explicitly. Let $Q_p^{(k,m)} = (u_1, u_2)^\top$ and $K_q^{(k,m)} = (v_1, v_2)^\top$. The 2D rotation matrix is:
\[
R_k(\Delta_k, m) = \begin{pmatrix} \cos(\theta_m^k \Delta_k) & -\sin(\theta_m^k \Delta_k) \\ \sin(\theta_m^k \Delta_k) & \cos(\theta_m^k \Delta_k) \end{pmatrix}
\]
Multiplying the rotation matrix by $K_q^{(k,m)}$:
\[
R_k(\Delta_k, m) K_q^{(k,m)} = \begin{pmatrix} v_1 \cos(\theta_m^k \Delta_k) - v_2 \sin(\theta_m^k \Delta_k) \\ v_1 \sin(\theta_m^k \Delta_k) + v_2 \cos(\theta_m^k \Delta_k) \end{pmatrix}
\]
Taking the inner product with $Q_p^{(k,m)}$:
\[
\begin{aligned}
\inner{Q_p^{(k,m)}}{R_k(\Delta_k, m) K_q^{(k,m)}} &= u_1 \left( v_1 \cos(\theta_m^k \Delta_k) - v_2 \sin(\theta_m^k \Delta_k) \right) \\
&\quad + u_2 \left( v_1 \sin(\theta_m^k \Delta_k) + v_2 \cos(\theta_m^k \Delta_k) \right)
\end{aligned}
\]
Grouping terms by $\cos(\theta_m^k \Delta_k)$ and $\sin(\theta_m^k \Delta_k)$:
\[
= \underbrace{(u_1 v_1 + u_2 v_2)}_{a_{p,m}^k} \cos(\theta_m^k \Delta_k) + \underbrace{(u_1 v_2 - u_2 v_1)}_{b_{p,m}^k} \sin(\theta_m^k \Delta_k)
\]
This gives the exact Fourier series form.

\paragraph{Step 4: Coefficient Bounds}
We prove the bounds for $a_{p,m}^k$ and $b_{p,m}^k$ using the Cauchy-Schwarz inequality. For $a_{p,m}^k$:
\[
|a_{p,m}^k| = |u_1 v_1 + u_2 v_2| = |\inner{Q_p^{(k,m)}}{K_q^{(k,m)}}| \leq \|Q_p^{(k,m)}\|_2 \|K_q^{(k,m)}\|_2
\]
For $b_{p,m}^k$, using the same method, we can obtain:
\[
|b_{p,m}^k| \leq \|Q_p^{(k,m)}\|_2 \|K_q^{(k,m)}\|_2
\]
This completes the proof.
\end{proof}

\begin{lemma}[Each Frequency Term Corresponds to a   $\le 2$ Matrix]
\label{lemma:Each Frequency Term Each Frequency Term}
For any fixed axis $k \in \{t,x,y\}$ and frequency $m$, the matrix $T_{k,m} \in \R^{L \times L}$ representing the $m$-th frequency interaction is defined by:
\[
T_{k,m}(p,q) = \inner{R_k(p_k, m) Q_p^{(k,m)}}{R_k(q_k, m) K_q^{(k,m)}}
\]
This matrix has rank at most 2.
\end{lemma}

\begin{proof}

\paragraph{Step 1: Define Rotated Query/Key Vectors}
For each token $p$, define the rotated 2D query vector:
\[
\hat{q}_{p,k,m} = R_k(p_k, m) Q_p^{(k,m)} \in \R^2
\]
For each token $q$, define the rotated 2D key vector:
\[
\hat{k}_{q,k,m} = R_k(q_k, m) K_q^{(k,m)} \in \R^2
\]
By definition, the $(p,q)$-th entry of $T_{k,m}$ is exactly the Euclidean inner product of these two vectors:
\[
T_{k,m}(p,q) = \inner{\hat{q}_{p,k,m}}{\hat{k}_{q,k,m}} = \hat{q}_{p,k,m}^\top \hat{k}_{q,k,m}
\]

\paragraph{Step 2: Construct Matrix Factorization}
We construct two matrices from the rotated vectors:
$U_{k,m} \in \R^{L \times 2}$: the $p$-th row of $U_{k,m}$ is $\hat{q}_{p,k,m}^\top$ (the transpose of the rotated query vector for token $p$).
$V_{k,m} \in \R^{L \times 2}$: the $q$-th row of $V_{k,m}$ is $\hat{k}_{q,k,m}^\top$ (the transpose of the rotated key vector for token $q$).

By construction, the transpose of $V_{k,m}$ is a $2 \times L$ matrix, where the $q$-th column is $\hat{k}_{q,k,m}$.

\paragraph{Step 3: Verify Matrix Product and Rank Bound}
We now compute the product $U_{k,m} V_{k,m}^\top$. The $(p,q)$-th entry of this product is the dot product of the $p$-th row of $U_{k,m}$ and the $q$-th column of $V_{k,m}^\top$, which is exactly:
\[
\left(U_{k,m} V_{k,m}^\top\right)_{p,q} = \hat{q}_{p,k,m}^\top \hat{k}_{q,k,m} = T_{k,m}(p,q)
\]
Thus, we have the exact matrix factorization:
\[
T_{k,m} = U_{k,m} V_{k,m}^\top
\]
We now apply Sylvester's Rank Inequality : for matrices $A \in \R^{m \times n}$ and $B \in \R^{n \times p}$, $\rank(AB) \leq \min(\rank(A), \rank(B))$. Here, $U_{k,m} \in \R^{L \times 2}$ and $V_{k,m}^\top \in \R^{2 \times L}$, so:
\[
\rank(T_{k,m}) = \rank(U_{k,m} V_{k,m}^\top) \leq \min(\rank(U_{k,m}), \rank(V_{k,m}^\top)) \leq 2
\]
This completes the proof.
\end{proof}

\begin{lemma}[Low-Rank Approximation of Background Pre-Softmax Matrix]
\label{lemma:Pre-Softmax Matrix rank}
Under Assumption \ref{assumption 1} and \ref{assumption 2}, for any error tolerance $\delta > 0$, there exists a truncated matrix $\tilde{S} \in \R^{L \times L}$ of rank at most:
\[
r = O\left(d_h \log\left(\frac{1}{\delta}\right)\right)
\]
such that the approximation error is bounded uniformly over the background set:
\[
\max_{(p,q) \in \OmegaBg} |s_{p,q} - \tilde{s}_{p,q}| \leq \delta
\]
\end{lemma}

\begin{proof}
\paragraph{Step 1: Decompose the Full Pre-Softmax Matrix}
From Lemma \ref{lemma 3D RoPE Induces Exact Fourier Series}, the full pre-softmax matrix $S$ (with entries $s_{p,q}$) can be written as a sum over axes and frequency terms:
\[
S = \sum_{k \in \{t,x,y\}} \sum_{m=1}^{d_k/2} T_{k,m}
\]
where each $T_{k,m}$ is the rank-$\leq 2$ matrix from Lemma \ref{lemma:Each Frequency Term Each Frequency Term}.

\paragraph{Step 2: Choose Truncation Threshold for Each Axis}
We aim to truncate the high-frequency terms such that the total tail error is bounded by $\delta$. We split the error equally across the three axes: we require the tail error for each axis to be at most $\delta/3$, so the total error is at most $\delta$.

Fix an axis $k$. From Assumption \ref{assumption 1}, the tail sum of coefficients decays exponentially:
\[
\sum_{m=m_0}^{d_k/2} \max_{(p,q) \in \OmegaBg} \left( |a_{p,m}^k| + |b_{p,m}^k| \right) \leq C_\omega \cdot \rho_k^{m_0}
\]
From Lemma \ref{lemma:Each Frequency Term Each Frequency Term}, each frequency term's entry is bounded by:
\[
|T_{k,m}(p,q)| = |a_{p,m}^k \cos(\theta_m^k \Delta_k) + b_{p,m}^k \sin(\theta_m^k \Delta_k)| \leq |a_{p,m}^k| + |b_{p,m}^k|
\]
since $|\cos(\cdot)| \leq 1$ and $|\sin(\cdot)| \leq 1$. Thus, the maximum entry of the tail sum for axis $k$ is bounded by:
\[
\max_{(p,q) \in \OmegaBg} \left| \sum_{m=m_0}^{d_k/2} T_{k,m}(p,q) \right| \leq \sum_{m=m_0}^{d_k/2} \max_{(p,q) \in \OmegaBg} |T_{k,m}(p,q)| \leq C_\omega \cdot \rho_k^{m_0}
\]
We require this tail error to be at most $\delta/3$:
\[
C_\omega \cdot \rho_k^{m_0} \leq \frac{\delta}{3}
\]
We solve for $m_0$. Since $\rho_k < 1$, $\log(\rho_k) < 0$. Rearranging the inequality:
\[
\rho_k^{m_0} \leq \frac{\delta}{3 C_\omega} \implies m_0 \log(\rho_k) \leq \log\left( \frac{\delta}{3 C_\omega} \right) \implies m_0 \geq \frac{\log\left( \frac{3 C_\omega}{\delta} \right)}{-\log(\rho_k)}
\]
Recall that $\rho_k = 10000^{-\alpha/d_k}$, so $-\log(\rho_k) = \frac{\alpha}{d_k} \log(10000)$. Substituting this in:
\[
m_0 \geq \frac{d_k \log\left( \frac{3 C_\omega}{\delta} \right)}{\alpha \log(10000)} = O\left( d_k \log\left( \frac{1}{\delta} \right) \right)
\]
We choose $M_k = \min\left( m_0, d_k/2 \right)$, which is the truncation index for axis $k$.

\paragraph{Step 3: Construct the Truncated Matrix}
We define the truncated pre-softmax matrix $\tilde{S}$ as the sum of the low-frequency terms:
\[
\tilde{S} = \sum_{k \in \{t,x,y\}} \sum_{m=1}^{M_k} T_{k,m}
\]
By construction, the approximation error is bounded by:
\[
\max_{(p,q) \in \OmegaBg} |s_{p,q} - \tilde{s}_{p,q}| = \max_{(p,q) \in \OmegaBg} \left| \sum_{k \in \{t,x,y\}} \sum_{m=M_k+1}^{d_k/2} T_{k,m}(p,q) \right| \leq \sum_{k \in \{t,x,y\}} \frac{\delta}{3} = \delta
\]
This satisfies the error bound requirement.

\paragraph{Step 4: Bound the Rank of the Truncated Matrix}
\label{lemma:FAVOR}
From Lemma \ref{lemma:Each Frequency Term Each Frequency Term}, each $T_{k,m}$ has rank at most 2. Using the rank of sum property, the rank of the truncated matrix is bounded by:
\[
\rank(\tilde{S}) \leq \sum_{k \in \{t,x,y\}} \sum_{m=1}^{M_k} \rank(T_{k,m}) \leq \sum_{k \in \{t,x,y\}} 2 M_k
\]
Substituting $M_k = O\left( d_k \log\left( \frac{1}{\delta} \right) \right)$ and $d_t + d_x + d_y = d_h$:
\[
\rank(\tilde{S}) \leq 2 \log\left( \frac{1}{\delta} \right) \sum_{k \in \{t,x,y\}} O(d_k) = O\left( d_h \log\left( \frac{1}{\delta} \right) \right)
\]
This completes the proof.
\end{proof}

\begin{lemma}[Positive Random Features for Exponential Kernel (FAVOR+)\cite{choromanski2021rethinking}, High-Probability Bound]
Let $\tilde{q}_i, \tilde{k}_j \in \R^r$ be vectors uniformly bounded by $\|\tilde{q}_i\|_2, \|\tilde{k}_j\|_2 \leq B$ for some constant $B>0$ independent of $L$. The exponential kernel $\exp(\tilde{q}_i^T \tilde{k}_j)$ can be approximated using a positive random feature map $\Phi: \R^r \to \R^R_+$. For any relative error tolerance $\delta_{\text{rel}} > 0$ and failure probability $\delta_{\text{fail}} > 0$, the required feature dimension is:
\[
R = O\left( \frac{r}{\delta_{\text{rel}}^2} \log\left( \frac{L^2}{\delta_{\text{fail}}} \right) \right)
\]
With probability at least $1 - \delta_{\text{fail}}$, provided $\exp(\tilde{q}_i^T\tilde{k}_j) \geq c > 0$ for all $i,j$ (this holds across the entire attention matrix, as the original query/key vectors are deterministically bounded by Lemma \ref{lemma Q/K Boundedness} and the low-rank truncation error is bounded by Lemma \ref{lemma:Pre-Softmax Matrix rank}, ensuring all approximated pre-softmax logits $\tilde{s}_{i,j}$ have a finite global minimum), the approximation satisfies the uniform relative error bound over all $L^2$ pairs:
\[
|\Phi(\tilde{q}_i)^T \Phi(\tilde{k}_j) - \exp(\tilde{q}_i^T \tilde{k}_j)| \leq \delta_{\text{rel}} \exp(\tilde{q}_i^T \tilde{k}_j)
\]
\end{lemma}

\begin{proof}

\paragraph{Step 1: Per-Pair FAVOR+ Guarantee via Hoeffding's Bound}
Based on the foundational properties of positive random features (FAVOR+) established in \cite{choromanski2021rethinking} (specifically Lemma 1 and Eq. 4), the random feature map $\Phi: \R^r \to \R^R_+$ provides an unbiased Monte Carlo estimator of the exponential kernel. For any fixed pair $(\tilde{q}_i, \tilde{k}_j)$, let the estimator be $\hat{E}_{i,j} = \Phi(\tilde{q}_i)^T \Phi(\tilde{k}_j)$. We have:
\[
\EE\left[ \hat{E}_{i,j} \right] = \exp(\tilde{q}_i^T \tilde{k}_j)
\]
Since the inputs are uniformly bounded ($\|\tilde{q}_i\|_2, \|\tilde{k}_j\|_2 \leq B$), the positive random features drawn from the FAVOR+ distribution are bounded by a constant $M = O(\sqrt{r})$. Therefore, $\hat{E}_{i,j}$ is an average of $R$ independent, bounded random variables. 

We apply Hoeffding's inequality to bound the \textbf{absolute error}. For any $\epsilon > 0$:
\[
\PP\left( |\hat{E}_{i,j} - \exp(\tilde{q}_i^T \tilde{k}_j)| > \epsilon \right) \leq 2 \exp\left( - \frac{2 R \epsilon^2}{M^2} \right)
\]
To obtain a \textbf{relative error} bound, we set the absolute error tolerance to $\epsilon = \delta_{\text{rel}} \exp(\tilde{q}_i^T \tilde{k}_j)$. Substituting this into the bound yields:
\[
\PP\left( |\hat{E}_{i,j} - \exp(\tilde{q}_i^T \tilde{k}_j)| > \delta_{\text{rel}} \exp(\tilde{q}_i^T \tilde{k}_j) \right) \leq 2 \exp\left( - \frac{2 R \delta_{\text{rel}}^2 \exp(\tilde{q}_i^T \tilde{k}_j)^2}{M^2} \right)
\]

Crucially, we must establish a positive lower bound for the squared exponential kernel $\exp(\tilde{q}_i^T \tilde{k}_j)^2 = \exp(2\tilde{s}_{i,j})$ to bound the failure probability. 
By Lemma \ref{lemma Q/K Boundedness}, the original pre-softmax logits are lower-bounded: $s_{i,j} \geq -Q_{\text{max}}^2$. 
By Lemma \ref{lemma:Pre-Softmax Matrix rank}, the truncation error is bounded by $\delta$, giving a global minimum for the approximated logits: $\tilde{s}_{i,j} \geq -Q_{\text{max}}^2 - \delta$. 

Because the exponential function is positive and monotonically increasing, the true exponential kernel is lower-bounded by a global positive constant $c$:
\[
\exp(\tilde{s}_{i,j}) \geq \exp(-Q_{\text{max}}^2 - \delta) \equiv c > 0
\]
Consequently, its square is also lower-bounded by a positive constant:
\[
\exp(\tilde{s}_{i,j})^2 \geq c^2 > 0
\]
We can therefore upper-bound the failure probability by replacing the squared exponential term with its minimal value $c^2$:
\[
\PP\left( \text{Relative Error} > \delta_{\text{rel}} \right) \leq 2 \exp\left( - \frac{2 R \delta_{\text{rel}}^2 c^2}{M^2} \right)
\]
To ensure this per-pair failure probability is at most $\frac{\delta_{\text{fail}}}{L^2}$ (which is required for the subsequent union bound), we enforce:
\[
2 \exp\left( - \frac{2 R \delta_{\text{rel}}^2 c^2}{M^2} \right) \leq \frac{\delta_{\text{fail}}}{L^2}
\]
Solving for the required feature dimension $R$:
\[
R \geq \frac{M^2}{2 c^2 \delta_{\text{rel}}^2} \log\left( \frac{2 L^2}{\delta_{\text{fail}}} \right)
\]
Since $M^2 = O(r)$ and $c$ is a positive constant independent of $L$, the required dimension to guarantee the relative error bound for a single pair with probability $1 - \frac{\delta_{\text{fail}}}{L^2}$ is:
\[
R = O\left( \frac{r}{\delta_{\text{rel}}^2} \log\left( \frac{L^2}{\delta_{\text{fail}}} \right) \right)
\]

\paragraph{Step 2: Uniform Guarantee via Union Bound}
We now extend the per-pair bound to a uniform bound over all $L^2$ query-key pairs. Let $E_{i,j}$ denote the event that the relative error for pair $(i,j)$ exceeds $\delta_{\text{rel}}$. From Step 1, $\PP(E_{i,j}) \leq \frac{\delta_{\text{fail}}}{L^2}$ for all $i,j$.

We apply the Union Bound to the union of all $L^2$ events:
\[
\PP\left( \bigcup_{i,j=1}^L E_{i,j} \right) \leq \sum_{i,j=1}^L \PP(E_{i,j}) \leq L^2 \cdot \frac{\delta_{\text{fail}}}{L^2} = \delta_{\text{fail}}
\]
Taking the complement of this event, we get:
\[
\PP\left( \bigcap_{i,j=1}^L E_{i,j}^c \right) \geq 1 - \delta_{\text{fail}}
\]
In other words, with probability at least $1 - \delta_{\text{fail}}$, the relative error is bounded by $\delta_{\text{rel}}$ for all $L^2$ pairs:
\[
|\Phi(\tilde{q}_i)^T \Phi(\tilde{k}_j) - \exp(\tilde{q}_i^T \tilde{k}_j)| \leq \delta_{\text{rel}} \exp(\tilde{q}_i^T \tilde{k}_j) \quad \forall i,j
\]
The total required feature dimension to satisfy this uniform bound across all pairs is $R = R_0 \cdot O\left( \log\left( \frac{L^2}{\delta_{\text{fail}}} \right) \right) = O\left( \frac{r}{\delta_{\text{rel}}^2} \log\left( \frac{L^2}{\delta_{\text{fail}}} \right) \right)$. This completes the proof.
\end{proof}

\begin{lemma}[Exact Diagonal Normalization and Rank Preservation]
\label{lemma:Exact Diagonal Normalization and Rank Preservation}
Let $Z_p = \sum_{r=1}^L \exp(s_{p,r})$ be the true row partition function, and define $D = \diag(1/Z_1, \dots, 1/Z_L)$. For any approximated exponentiated kernel matrix $\hat{E} \in \R^{L \times L}$, the normalized matrix $\hat{A} = D \hat{E}$ satisfies:
\[
\rank(\hat{A}) = \rank(\hat{E})
\]
\end{lemma}
\begin{proof}

\paragraph{Step 1: Prove $D$ is Invertible}
By definition, the partition function $Z_p = \sum_{r=1}^L \exp(s_{p,r})$. Since the exponential function is positive ($\exp(x) > 0$ for all $x \in \R$), every term in the sum is positive, so $Z_p > 0$ for all rows $p$.

The diagonal matrix $D$ has entries $1/Z_p$ on its diagonal, all of which are positive. A diagonal matrix is invertible if and only if all its diagonal entries are non-zero. Thus, $D$ is full-rank and invertible, with inverse $D^{-1} = \diag(Z_1, Z_2, \dots, Z_L)$.

\paragraph{Step 2: Prove Rank Preservation}
For any matrix $\hat{E} \in \R^{L \times L}$ and invertible matrix $D \in \R^{L \times L}$, left-multiplying by $D$ preserves the rank:
\[
\rank(D \hat{E}) = \rank(\hat{E})
\]
Since $\hat{A} = D \hat{E}$, we have:
\[
\rank(\hat{A}) = \rank(\hat{E})
\]
This completes the proof.
\end{proof}

\begin{corollary}
[Constructive Sparse-Low-Rank Reconstruction of Attention Matrix]
\label{corollary: main}
Let $\mathcal{E} > 0$ be a fixed global absolute error tolerance, and let $\delta_{\text{fail}} = O(L^{-c})$ for some constant $c>0$. For any energy threshold $\tau \in [2\mathcal{E}, 1)$, we can construct an approximate reconstructed attention matrix $\hat{A}_{\text{final}} = \tilde{A}_{\text{sparse}} + \hat{A}_{\text{lowrank}}$, where:
\begin{enumerate}
    \item The \textbf{residual sparse branch} $\tilde{A}_{\text{sparse}}$ is sparse, with exactly the same support as the original spike set $\OmegaSpike$, and thus has at most $\lfloor 1/\tau \rfloor$ non-zero entries per row;
    \item The \textbf{low-rank branch} $\hat{A}_{\text{lowrank}}$ is a globally defined, unmasked matrix that acts as a low-rank approximator for the background target $A_{\text{bg}}$. It operates with a provable bottleneck rank bound:
    \[
    {R = O\left( d_h \cdot \left(\frac{\tau}{\mathcal{E}}\right)^2 \cdot \log\left( \frac{\tau}{\mathcal{E}} \right) \cdot \log L \right)}
    \]
\end{enumerate}
With probability at least $1 - \delta_{\text{fail}}$, the low-rank approximator tightly bounds the background target on $\OmegaBg$:
\[
\max_{(p,q) \in \OmegaBg} |\hat{A}_{\text{lowrank}}(p,q) - A_{\text{bg}}(p,q)| \leq \mathcal{E}
\]
Consequently, the final reconstructed matrix $\hat{A}_{\text{final}}$ satisfies the global uniform error guarantee across all $L^2$ pairs:
\[
\max_{p,q} |\hat{A}_{\text{final}}(p,q) - A(p,q)| \leq \mathcal{E}
\]
\end{corollary}

\begin{proof}
\paragraph{Step 1: Pre-Softmax Low-Rank Approximation}
We first choose the pre-softmax absolute error tolerance and FAVOR+ relative error tolerance to be equal: $\delta = \delta_{\text{rel}} = \frac{\mathcal{E}}{4\tau}$. This choice is derived from detailed error propagation analysis, ensuring that the compounded error after exponential kernel linearization and normalization results in the desired absolute error bound $\mathcal{E}$ on the background set.

Since $\tau \in [2\mathcal{E}, 1)$, we have $\frac{\mathcal{E}}{\tau} \leq 0.5$, so $\delta \leq 1/8$, which satisfies the validity condition for the exponential inequalities used in error propagation.
Since $\tau \in [2\mathcal{E}, 1)$, we have $\frac{\mathcal{E}}{\tau} \leq 0.5$, so $\delta \leq 1/8$, which is a valid small error tolerance.

By Lemma \ref{lemma:Pre-Softmax Matrix rank}, there exists a truncated pre-softmax matrix $\tilde{S}$ of rank:
\[
r = O\left( d_h \log\left( \frac{1}{\delta} \right) \right) = O\left( d_h \log\left( \frac{4\tau}{\mathcal{E}} \right) \right) = O\left( d_h \log\left( \frac{\tau}{\mathcal{E}} \right) \right)
\]
such that the approximation error is bounded uniformly over the background set:
\[
\max_{(p,q) \in \OmegaBg} |s_{p,q} - \tilde{s}_{p,q}| \leq \delta = \frac{\mathcal{E}}{4\tau}
\]
Since $\tilde{S}$ has rank $r$, by the definition of matrix rank, we can factor it as:
\[
\tilde{S} = \tilde{Q} \tilde{K}^\top, \quad \text{where } \tilde{Q}, \tilde{K} \in \R^{L \times r}
\]

\paragraph{Step 2: Set FAVOR+ Relative Error Tolerance}
We choose the FAVOR+ relative error tolerance to match the pre-softmax error:
\[
\delta_{\text{rel}} = \frac{\mathcal{E}}{4\tau}
\]
and set the failure probability to $\delta_{\text{fail}} = O(L^{-c})$, as specified in the following theorem statement.

\paragraph{Step 3: FAVOR+ Linearization with Exact Rank Bound}
We apply Lemma \ref{lemma:FAVOR} to the low-dimensional embeddings $\tilde{q}_i$ (rows of $\tilde{Q}$) and $\tilde{k}_j$ (rows of $\tilde{K}$). From Lemma \ref{lemma Q/K Boundedness}, the original query/key vectors are bounded, so the low-dimensional embeddings are also uniformly bounded by a constant $B>0$ independent of $L$. For the background set, $\exp(\tilde{q}_i^T \tilde{k}_j) = \exp(\tilde{s}_{i,j})$ is bounded below by a constant $c>0$ (since $\tilde{s}_{i,j}$ is bounded for the smooth background).

The required feature dimension is:
\[
R = O\left( \frac{r}{\delta_{\text{rel}}^2} \log\left( \frac{L^2}{\delta_{\text{fail}}} \right) \right)
\]
We substitute the two core terms:
1.  Pre-softmax truncation rank: $r = O\left( d_h \log\left( \frac{\tau}{\mathcal{E}} \right) \right)$;
2.  Relative error tolerance: $\delta_{\text{rel}} = \frac{\mathcal{E}}{4\tau}$.

Substituting these into the rank bound:
\[
R = O\left( \frac{d_h \log\left( \frac{\tau}{\mathcal{E}} \right)}{\left( \frac{\mathcal{E}}{4\tau} \right)^2} \log\left( \frac{L^2}{\delta_{\text{fail}}} \right) \right)
\]
Critically, since $\delta_{\text{fail}} = O(L^{-c})$, we have:
\[
\log\left( \frac{L^2}{\delta_{\text{fail}}} \right) = O(\log L)
\]
We thus obtain the final, simplified rank bound:
\[
R = O\left( d_h \cdot \left(\frac{\tau}{\mathcal{E}}\right)^2 \cdot \log\left( \frac{\tau}{\mathcal{E}} \right) \cdot \log L \right)
\]
Let $\hat{E}_{i,j} = \Phi(\tilde{q}_i)^T \Phi(\tilde{k}_j)$ be the approximated exponential kernel matrix. By Lemma \ref{lemma:FAVOR}, with probability at least $1 - \delta_{\text{fail}}$, the approximation satisfies:
\[
|\hat{E}_{i,j} - \exp(\tilde{s}_{i,j})| \leq \delta_{\text{rel}} \exp(\tilde{s}_{i,j}) = \frac{\mathcal{E}}{4\tau} \exp(\tilde{s}_{i,j})
\]
By construction, $\rank(\hat{E}) \leq R$, since it is the product of two $L \times R$ matrices.

\paragraph{Step 4: Error Propagation via Exponential Inequalities}
We now propagate the pre-softmax approximation error and FAVOR+ error to the final post-softmax attention matrix. Since $\tau \in [2\mathcal{E}, 1)$, we have $\frac{\mathcal{E}}{\tau} \leq 0.5$, so all exponential inequalities apply.

First, we bound the error between $\exp(\tilde{s}_{p,q})$ and $\exp(s_{p,q})$. From Step 1, we have $|\tilde{s}_{p,q} - s_{p,q}| \leq \frac{\mathcal{E}}{4\tau} \leq 0.125$. Using the exponential inequalities:
\[
\exp\left(s_{p,q} - \frac{\mathcal{E}}{4\tau}\right) \leq \exp(\tilde{s}_{p,q}) \leq \exp\left(s_{p,q} + \frac{\mathcal{E}}{4\tau}\right)
\]
Applying the bounds $e^x \leq 1+2x$ and $e^{-x} \geq 1-x$ for $x \leq 0.5$:
\[
\exp(s_{p,q}) \left(1 - \frac{\mathcal{E}}{4\tau}\right) \leq \exp(\tilde{s}_{p,q}) \leq \exp(s_{p,q}) \left(1 + 2 \cdot \frac{\mathcal{E}}{4\tau}\right) = \exp(s_{p,q}) \left(1 + \frac{\mathcal{E}}{2\tau}\right)
\]

Next, we combine this with the FAVOR+ error bound from Step 3. The FAVOR+ bound can be rewritten as:
\[
\exp(\tilde{s}_{p,q}) \left(1 - \frac{\mathcal{E}}{4\tau}\right) \leq \hat{E}_{p,q} \leq \exp(\tilde{s}_{p,q}) \left(1 + \frac{\mathcal{E}}{4\tau}\right)
\]

Substituting the bounds for $\exp(\tilde{s}_{p,q})$ into this inequality, we get the compounded error bound for $\hat{E}_{p,q}$ relative to the true exponential kernel $\exp(s_{p,q})$:
\[
\exp(s_{p,q}) \left(1 - \frac{\mathcal{E}}{4\tau}\right)^2 \leq \hat{E}_{p,q} \leq \exp(s_{p,q}) \left(1 + \frac{\mathcal{E}}{2\tau}\right)\left(1 + \frac{\mathcal{E}}{4\tau}\right)
\]

We expand the upper bound product explicitly:
\[
\left(1 + \frac{\mathcal{E}}{2\tau}\right)\left(1 + \frac{\mathcal{E}}{4\tau}\right) = 1 + \frac{\mathcal{E}}{2\tau} + \frac{\mathcal{E}}{4\tau} + \frac{\mathcal{E}^2}{8\tau^2} = 1 + \frac{3\mathcal{E}}{4\tau} + \frac{\mathcal{E}^2}{8\tau^2}
\]
Since $\frac{\mathcal{E}}{\tau} \leq 0.5$, the quadratic term is bounded by:
\[
\frac{\mathcal{E}^2}{8\tau^2} = \frac{\mathcal{E}}{8\tau} \cdot \frac{\mathcal{E}}{\tau} \leq \frac{\mathcal{E}}{8\tau} \cdot 0.5 = \frac{\mathcal{E}}{16\tau}
\]
Substituting this back, the upper bound becomes:
\[
1 + \frac{3\mathcal{E}}{4\tau} + \frac{\mathcal{E}}{16\tau} = 1 + \frac{13\mathcal{E}}{16\tau} < 1 + \frac{\mathcal{E}}{\tau}
\]

For the lower bound, we expand the square:
\[
\left(1 - \frac{\mathcal{E}}{4\tau}\right)^2 = 1 - \frac{\mathcal{E}}{2\tau} + \frac{\mathcal{E}^2}{16\tau^2} \geq 1 - \frac{\mathcal{E}}{2\tau} > 1 - \frac{\mathcal{E}}{\tau}
\]

Combining the upper and lower bounds, we get the final relative error bound for the exponential kernel:
\[
\left| \hat{E}_{p,q} - \exp(s_{p,q}) \right| \leq \frac{\mathcal{E}}{\tau} \exp(s_{p,q})
\]

Finally, we normalize by the true row partition function $Z_p$ to get the attention matrix. Define the low-rank attention matrix:
\[
\hat{A}_{\text{lowrank}}(p,q) = \frac{\hat{E}_{p,q}}{Z_p}
\]
Dividing both sides of the relative error bound by $Z_p$, and using the definition of the true attention weight $A(p,q) = \frac{\exp(s_{p,q})}{Z_p}$, we get:
\[
| \hat{A}_{\text{lowrank}}(p,q) - A(p,q) | \leq \frac{\mathcal{E}}{\tau} A(p,q)
\]
For $(p,q) \in \OmegaBg$, we have $A(p,q) \leq \tau$ by definition. Substituting this yields the absolute error bound on the background set:
\[
| \hat{A}_{\text{lowrank}}(p,q) - A(p,q) | \leq \mathcal{E}
\]
Crucially, by Definition \ref{def:1}, the true background target $A_{\text{bg}}(p,q) = A(p,q)$ for all $(p,q) \in \OmegaBg$. Therefore, we have formally established that our globally unmasked low-rank branch tightly approximates the background target:
\[
\max_{(p,q) \in \OmegaBg} |\hat{A}_{\text{lowrank}}(p,q) - A_{\text{bg}}(p,q)| \leq \mathcal{E}
\]
with probability at least $1 - \delta_{\text{fail}}$.

\paragraph{Step 5: Residual Reconstruction}
We emphasize that $\hat{A}_{\text{lowrank}}$ is defined on the entire image. Formally:
\[
\hat{A}_{\text{lowrank}} = D \hat{E}, \quad \text{where } D = \diag(1/Z_1, \dots, 1/Z_L)
\]
From Lemma \ref{lemma:Exact Diagonal Normalization and Rank Preservation}, left-multiplying by the invertible diagonal matrix $D$ preserves rank:
\[
\rank(\hat{A}_{\text{lowrank}}) = \rank(\hat{E}) \leq R = O\left( d_h \left(\frac{\tau}{\mathcal{E}}\right)^2 \log\left( \frac{\tau}{\mathcal{E}} \right) \log L \right)
\]
We use the \textbf{residual reconstruction} defined in Definition \ref{def:2} to preserve low-rankness:
\[
\hat{A}_{\text{final}} = \tilde{A}_{\text{sparse}} + \hat{A}_{\text{lowrank}}
\]

We now prove the global uniform error guarantee by splitting the analysis into the two disjoint sets $\OmegaSpike$ and $\OmegaBg$:

\textbf{Case 1: $(p,q) \in \OmegaSpike$ (High-Energy Spike Set).} 
By Definition \ref{def:2}, the residual sparse branch acts as an exact compensator for the low-rank output: $\tilde{A}_{\text{sparse}}(p,q) = A(p,q) - \hat{A}_{\text{lowrank}}(p,q)$. Thus, the final reconstructed matrix guarantees zero error on the spike set:
\[
\hat{A}_{\text{final}}(p,q) = \left( A(p,q) - \hat{A}_{\text{lowrank}}(p,q) \right) + \hat{A}_{\text{lowrank}}(p,q) = A(p,q) \implies |\hat{A}_{\text{final}}(p,q) - A(p,q)| = 0
\]

\textbf{Case 2: $(p,q) \in \OmegaBg$ (Low-Energy Background Set).} 
By Definition \ref{def:2}, the sparse branch is masked out ($\tilde{A}_{\text{sparse}}(p,q) = 0$). Thus, the final reconstructed matrix relies entirely on the low-rank approximator:
\[
\hat{A}_{\text{final}}(p,q) = \hat{A}_{\text{lowrank}}(p,q)
\]
Since $A(p,q) = A_{\text{bg}}(p,q)$ on this set, we directly apply the bound established at the end of Step 4:
\[
|\hat{A}_{\text{final}}(p,q) - A(p,q)| = |\hat{A}_{\text{lowrank}}(p,q) - A_{\text{bg}}(p,q)| \leq \mathcal{E}
\]

\textbf{Global Error Bound.} 
Combining both cases, the maximum error is exactly bounded by $\mathcal{E}$ everywhere:
\[
\max_{p,q} |\hat{A}_{\text{final}}(p,q) - A(p,q)| \leq \mathcal{E}
\]
\paragraph{Step 6: Conditional Monotonicity and Asymptotic Scaling of the Rank Bound}
We now prove the conditional monotonicity of the required rank $R$ with respect to $\tau$, and analyze its asymptotic behavior. The rank bound is:
\[
R(\tau) = C \cdot d_h \left(\frac{\tau}{\mathcal{E}}\right)^2 \log\left( \frac{\tau}{\mathcal{E}} \right) \log L
\]
where $C>0$ is a constant independent of $\tau$. We analyze the behavior of the function $f(x) = x^2 \log x$ for $x > 0$. The derivative is:
\[
f'(x) = 2x \log x + x^2 \cdot \frac{1}{x} = x (2 \log x + 1)
\]
For $x \geq 2$ (since $\tau \geq 2\mathcal{E} \implies x = \tau/\mathcal{E} \geq 2$), $\log x \geq \log 2 > 0$, so $f'(x) > 0$. Thus, $f(x)$ is strictly increasing for $x \geq 2$.

\end{proof}
\subsection{Main theorem and result analysis}

\begin{limbox}
\begin{theorem}
\label{theorem: main} 
    Under Assumption \ref{assumption 1} and \ref{assumption 2} (empirically validated bounded projection weights and exponential spectral concentration of 3D RoPE), for any sequence length $L$, choosing the energy threshold $\tau$ and error tolerance $\mathcal{E}$ as:
\[
\tau = \frac{c}{ L^{\frac{1}{2}}}, \quad \mathcal{E} = \frac{\tau}{2} = \frac{c}{ 2L^{\frac{1}{2}}}
\]
where $c > 0$ is a constant independent of $L$, the post-softmax attention matrix $A$ inherently decomposes into a high-energy sparse target ($A_{\text{sparse}}$) and a smooth background target ($A_{\text{bg}}$). Consequently, there exists a sparse-low-rank reconstructed matrix $\hat{A}_{\text{final}} = \tilde{A}_{\text{sparse}} + \hat{A}_{\text{lowrank}}$ satisfying the following properties:
\begin{enumerate}
    \item \textbf{Sub-quadratic Sparsity:} The residual sparse branch $\tilde{A}_{\text{sparse}}$, which acts as an exact error compensator on the spike set ($\tilde{A}_{\text{sparse}}(p,q) = A(p,q) - \hat{A}_{\text{lowrank}}(p,q)$ for $(p,q) \in \Omega_\tau$), has total non-zero entries bounded by $\text{NNZ}(\tilde{A}_{\text{sparse}}) = \mathcal{O}(L^{\frac{3}{2}})$.
    \item \textbf{Sub-linear Rank:} The globally unmasked low-rank branch $\hat{A}_{\text{lowrank}}$, which acts as a dense approximator for the background target $A_{\text{bg}}$, requires a bottleneck rank of $R = \mathcal{O}\left( d_h \cdot \log L \right)$.
    \item \textbf{Asymptotic Error Bound:} With high probability $1 - \delta_{\text{fail}}$, the global reconstruction error is uniformly bounded by $\max_{p,q} |\hat{A}_{\text{final}}(p,q) - A(p,q)| = \mathcal{O}\left( L^{-\frac{1}{2}} \right)$.
\end{enumerate}
\end{theorem}
\end{limbox}

\begin{proof}
We prove each statement sequentially:
\paragraph{Proof of Statement 1 (Sparse Branch Complexity)}
Given $\tau = \frac{c}{ L^{\frac{1}{2}}}$, by Theorem \ref{theorem 1}, the total number of non-zero entries is:
\[
\text{NNZ}(\tilde{A}_{\text{sparse}}) \leq L \cdot \left\lfloor \frac{1}{\tau} \right\rfloor = L \cdot \left\lfloor \frac{L^{\frac{1}{2}}}{c} \right\rfloor = O(L^{\frac{3}{2}})
\]

\paragraph{Proof of Statement 2 (Low-Rank Branch Rank)}
Given $\mathcal{E} = \frac{\tau}{2}$, we have $\frac{\tau}{\mathcal{E}} = 2$. Substituting into the rank bound from the main theorem:
\[
R = O\left( d_h \cdot \left(2\right)^2 \cdot \log\left( 2 \right) \cdot \log L \right)
\]
Thus:
\[
R = O\left( d_h\cdot logL \right)
\]
The rank decreases as sequence length increases.

\paragraph{Proof of Statement 3 (Global Reconstruction Error)}
By theorem \ref{theorem: main}, the global reconstruction error is bounded by the tolerance $\mathcal{E}$. Substituting $\mathcal{E}= \frac{c}{ 2L^{\frac{1}{2}}}$, we obtain the high-probability upper bound:
\begin{equation}
    \max_{p,q} |\hat{A}_{\text{final}}(p,q) - A(p,q)| \leq \frac{c}{ 2L^{\frac{1}{2}}} = O\left( L^{-\frac{1}{2}}\right)
\end{equation}
\end{proof}

\begin{remark}[Relation to Stable Rank]
\label{remark:stable rank}
Our theorem guarantees $A_{\text{bg}} = \hat{A}_{\text{lowrank}} + E$, where $\rank(\hat{A}_{\text{lowrank}}) \leq R$ and $\|E\|_\infty \leq \mathcal{E}$. Because the residual error $E$ is uniformly bounded element-wise by a small tolerance $\mathcal{E}$, the energy of the background matrix is overwhelmingly concentrated in the top $R$ singular components. This structural forces the stable rank of $A_{\text{bg}}$ to be heavily dominated by $R$.
\end{remark}

\begin{remark}[Asymptotic Superiority for Long-Sequence Inference]
\label{remark:long_sequence_scaling}
Theorem \ref{theorem: main} provides the fundamental mathematical justification for why RoPeSLR is exceptionally suited for long-sequence video generation. Consider the asymptotic behavior as sequence length $L \to \infty$:
\begin{enumerate}[leftmargin=*, itemsep=0em]
    \item \textbf{Vanishing Sparse Overhead}: The ratio of preserved high-energy sparse computations relative to full attention scales as $\frac{\mathcal{O}(L^{3/2})}{\mathcal{O}(L^2)} = \mathcal{O}(L^{-1/2})$. As $L$ approaches infinity, this ratio converges to 0. This guarantees that longer sequences naturally permit exponentially higher sparsity ratios without sacrificing vital semantic spikes.
    \item \textbf{Sub-linear Rank Bottleneck}: The rank required to closely capture the global background continuum grows only logarithmically ($R = \mathcal{O}(\log L)$). Compared to the rapid scaling of the sequence itself, this dictates that the low-rank branch becomes increasingly compact and heavily compressed as video length expands.
\end{enumerate}
Together, these bounds prove that RoPeSLR actively benefits from sequence scaling, bypassing the $\mathcal{O}(L^2)$ curse and establishing an asymptotically optimal theoretical foundation for ultra-long video synthesis.
\end{remark}

\begin{remark}[The Indispensable Role of 3D RoPE in This Proof]
\label{remark:3D ROPE}
The tight sequence-length-independent rank bound we derive is \textbf{exclusively enabled by the mathematical structure of 3D RoPE}, with two non-substitutable theoretical contributions directly grounded in our proofs:
\begin{enumerate}
    \item It enables an \textbf{exact Fourier series expansion} of the pre-softmax QK inner product over relative positions (\textbf{Lemma \ref{lemma 3D RoPE Induces Exact Fourier Series}}), which decomposes the full attention matrix into a sum of rank-$\le 2$ frequency-specific matrices (\textbf{Lemma \ref{lemma:Each Frequency Term Each Frequency Term}}).
    \item Its standard exponential frequency schedule induces \textbf{exponential decay of high-frequency interaction coefficients} (empirically validated in Section \ref{sec:empirical_validation}). Leveraging this property, \textbf{Lemma \ref{lemma:Pre-Softmax Matrix rank}} proves that the background matrix can be approximated with a rank bound of $O(d_h \log(1/\delta))$, which is independent of sequence length $L$.
\end{enumerate}
\end{remark}

\begin{remark}[Understanding the Token-Wise Context Compensator: Core Insights]
\label{remark:token_wise_derivation}
To bridge the logical gap between matrix-level low-rankness (Theorem \ref{theorem: main}) and point-wise implementation, we frame the compensator as an \textbf{amortized implicit proxy}.

\textbf{1. Algebraic Factorization of Global Aggregation:}
While trigonometric identities explicitly decouple the pre-softmax 3D RoPE logits into absolute spatial bases, the non-linear softmax operation typically breaks this pairwise separation. However, because Theorem \ref{theorem: main} guarantees that the background continuum $A_{\text{bg}}$ is exceptionally low-rank, it admits a low-dimensional matrix factorization. Following the principles of kernelized linear attention \cite{choromanski2021rethinking, wang2020linformer}, this low-rank post-softmax matrix can be decomposed into $A_{\text{bg}}(p,q) \approx \sum_{k=1}^r \phi_k(p) \psi_k(q)$, where $\phi_k$ and $\psi_k$ denote the $k$-th dimension of the decoupled spatial embeddings. Consequently, the global aggregation $(A_{\text{bg}}V)_p$ at position $p$ factorizes as:
\begin{equation}
    (A_{\text{bg}}V)_p = \sum_{q=1}^L A_{\text{bg}}(p,q) V_q \approx \sum_{k=1}^r \phi_k(p) \underbrace{\left( \sum_{q=1}^L \psi_k(q) V_q \right)}_{\text{Global Context } C_k}
\end{equation}
This derivation formally proves that the smooth background context is essentially a linear combination of spatial bases $\phi_k(p)$, weighted by global coefficients $C_k$ that are entirely invariant to the query position $p$.

\textbf{2. Amortized Decoding via Dense State Encapsulation:}
While $C_k$ represents a global sequence-level sum, it does not require explicit re-computation if the token $X_p$ already "perceives" the macroscopic state. In deep DiTs, $X_p$ is a dense semantic capsule: it has accumulated global context via preceding self-attention layers and is strictly modulated by global timestep and prompt embeddings. Thus, $X_p$ inherently encapsulates the descriptors $C_k$. By combining $X_p$ with explicit 3D positional bases ($\text{PE}_{\text{3D}}$), the MLP functions as a conditional decoder, evaluating the spatial field point-wise. This is supported by \textbf{spectral bias} theory \cite{rahaman2019spectral}, which suggests MLPs are naturally biased toward learning such low-frequency global manifolds.

\textbf{3. Empirical Verification:}
The structural validity of this proxy is confirmed by our Stage-I alignment objective $\mathcal{L}_{\text{align}} = \| O_{\text{total}} - A V \|_F^2$, which converges rapidly to a minimal floor (\textbf{Figure \ref{fig:training_loss}}). Mechanistically, our \textbf{Gram spectral analysis (Appendix \ref{subsec:svd_mechanistic_analysis})} proves that the MLP's output eigenvectors perfectly match the geometric standing waves of the ground-truth background, confirming that the learned proxy successfully captures the theoretically predicted manifold.
\end{remark}

\begin{remark}[Superiority of Sparse-Low-Rank Structure in DiTs and LLMs]
\label{remark: LLM and DIT}
Our framework proves that DiTs possess a more pronounced sparse-low-rank nature than LLMs, grounded in two mathematical inequalities:
\begin{enumerate}
    \item \textbf{Cubic Acceleration of Spectral Decay (Rank Inequality)}:
    As derived in Lemma \ref{lemma:Pre-Softmax Matrix rank}, the truncation index $M_k$ required to bound the error depends on the exponential decay constant $\rho$. 
    For LLMs using 1D RoPE with head dimension $d_h$, the decay rate is $\rho_{\text{LLM}} = 10000^{-\alpha/d_h}$.
    For DiTs using 3D RoPE, the dimension is partitioned ($d_t, d_x, d_y$). Assuming a uniform partition $d_k = d_h / 3$, the per-axis decay rate is $\rho_{\text{DiT}} = 10000^{-\alpha/(d_h/3)} = \rho_{\text{LLM}}^3$.
    
    Recall from the proof of Lemma \ref{lemma:Pre-Softmax Matrix rank} that the required rank $R$ to achieve an error $\delta$ scales inversely with $-\log(\rho)$. Therefore, we have the theoretical relationship:
    \begin{equation}
        R_{\text{DiT}} \propto \frac{1}{-\log(\rho_{\text{DiT}})} = \frac{1}{-3 \log(\rho_{\text{LLM}})} = \frac{1}{3} R_{\text{LLM}}
    \end{equation}
    This guarantees that \textbf{under the exact same error tolerance, the stable rank of the background matrix in DiTs is asymptotically one-third that of LLMs}, massively amplifying the efficiency of the low-rank branch.
    
   \item \textbf{Empirical Tightness of the Deterministic Sparsity Bound}: 
    While the theoretical $\mathcal{O}(L^{3/2})$ sparsity upper bound (derived via $\text{NNZ} \leq L \cdot \lfloor 1/\tau \rfloor$) is a deterministic property holding for any row-stochastic attention matrix, its practical utility heavily depends on the data domain. In LLMs, tasks like in-context learning scatter semantic spikes across the causal history, causing this universal bound to be vacuous in practice. Conversely, the inherent spatiotemporal locality of video naturally clusters high-energy semantic interactions. This physical structure ensures that the theoretical deterministic bound aligns tightly with the empirical attention manifold in DiTs, rendering the $\mathcal{O}(L^{3/2})$ extraction practically effective.
    \item \textbf{Empirical Evidence}: 
We provide experimental evidence to validate this remark, see figure \ref{fig:attn_decomp_hunyuan} and \ref{fig:attn_decomp_qwen} for more details.
\end{enumerate}
\end{remark}

\newpage
\section{Additional Experiments}
\label{sec:Additional Experiments}
\subsection{Empirical Validation of Spectral Concentration (Assumption \ref{assumption 2})}
\label{sec:empirical_validation}

\subsubsection{Spectral Evolution and Justification of Exponential Decay}
\label{subsec:spectral_evolution}

To verify the exponential decay $\rho_k^{m_0}$ ($\rho_k < 1$) proposed in Assumption \ref{assumption 2}, we measure the contribution of each frequency index $m$ for axis $k \in \{t, x, y\}$. We define the \textbf{Maximum Interaction Magnitude} $\mathcal{M}(k, m)$ using the $99^{\text{th}}$ percentile:
\begin{equation}
    \mathcal{M}(k, m) = \text{Percentile}_{99\%} \left( \left| \langle Q_p^{(k,m)}, R_k(\Delta) K_q^{(k,m)} \rangle \right| \right)
\end{equation}
Using a strict maximum over all $L^2$ pairs is vulnerable to isolated sparse spikes. Filtering the top 1\% extreme interactions provides a statistically robust proxy for the true supremum of the smooth background continuum \cite{wang2020linformer}. We simultaneously compute the \textbf{Cumulative Tail Energy}: $\text{Tail}(m_0) = \sum_{m=m_0}^{d_k/2} \mathcal{M}(k, m)$.

\textbf{Understanding the Visualizations.} We evaluate these metrics on Wan2.1-T2V-1.3B, Wan2.1-T2V-14B, and HunyuanVideo-13B across multiple denoising stages. In Figure \ref{fig:denoising_evolution}, the horizontal axes denote the RoPE dimension index $m$ (larger $m$ corresponds to a lower rotational frequency). The vertical axes are on a \textbf{logarithmic scale} ($\log_{10}$). An exponential decay $f(m) = C \cdot \rho^m$ translates to a linear downward trend on a semi-log plot.

\textbf{Key Observations \& Generalization.} Figure \ref{fig:denoising_evolution} shows that across \textbf{all three models}, the interaction energy along the $t, h,$ and $w$ dimensions consistently exhibits a clear decline on the Log-Y axes as $m$ increases. This confirms that the exponential spectral concentration is a universal structural property of 3D RoPE-equipped DiTs, rooted in RoPE's long-term decay prior \cite{su2021roformer}. Furthermore, as denoising progresses from Gaussian noise to semantic generation, this decay steepens (i.e., $\rho_k$ decreases). This comprehensively justifies Assumption \ref{assumption 2}.

\begin{figure}[htbp]
    \centering
    \includegraphics[width=1\textwidth]{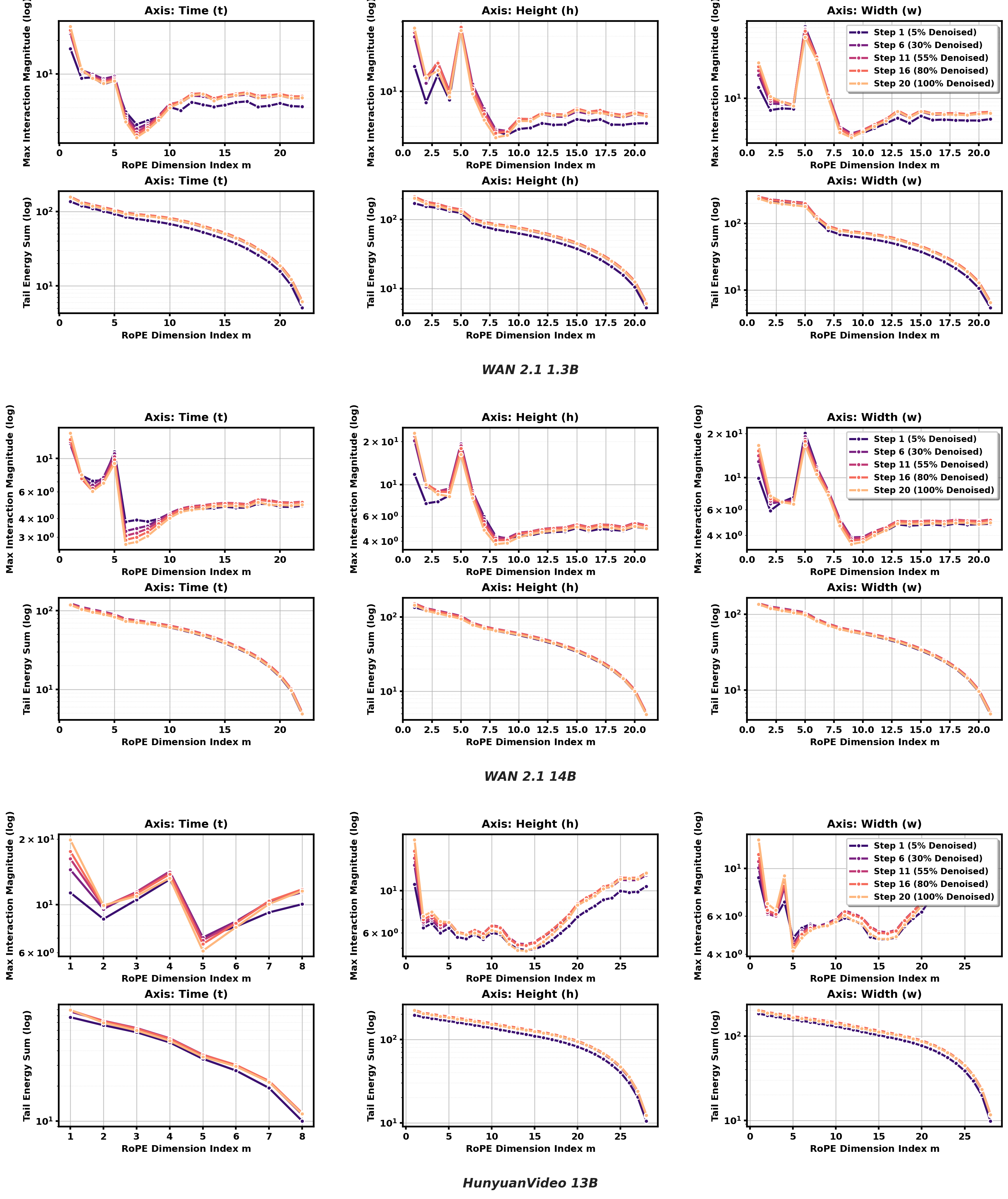}
    \caption{\textbf{Evolution of QK Spectral Concentration Driven by 3D RoPE.} We evaluate three pre-trained DiTs (Wan2.1-1.3B, Wan2.1-14B, and HunyuanVideo-13B) using 3 representative text prompts at $480 \times 832$ resolution, 49 frames. The X-axis is the RoPE dimension index $m \in [1, d_k/2]$. We display the maximum interaction magnitude (approximated via the $99^{\text{th}}$ percentile to filter sparse outliers and evaluate the background) and cumulative tail energy on a \textbf{Log-Y scale}. The consistent downward linear trends across models validate the structural exponential decay formalized in Assumption \ref{assumption 2}.}
    \label{fig:denoising_evolution}
\end{figure}

\subsection{Ablation Study}

\subsubsection{Ablation Study on Sparsity }
\label{subsec:sparsity_ablation}
We perform an ablation study on sparsity. Since the threshold scheme of VMoBA cannot precisely control the exact sparsity ratio, all sparsity values in the experiments are approximate results with an error within 1\%. Results in Table~\ref{tab:sparsity_ablation} show that increasing sparsity from 85\% to 90\% induces nearly no quality loss, while a further rise from 90\% to 95\% causes obvious performance degradation. The sparsity of 90\% is therefore a favorable sweet spot.

\begin{table}[!htbp]
\centering
\fontsize{9}{10}\selectfont   
\setlength{\tabcolsep}{2.5pt}
\caption{Ablation Study on Sparsity.}
\label{tab:sparsity_ablation}
\begin{tabular}{lccccc}
\hline
Sparsity & sc & bc & ms & aq & iq \\
\hline
80\% & 0.9607 & 0.9502 & 0.9899 & 0.6075 & 0.6475 \\
90\% & 0.9510 & 0.9487 & 0.9898 & 0.6071 & 0.6440 \\
95\% & 0.9387 & 0.9127 & 0.9763 & 0.5870& 0.6193 \\
\hline
\end{tabular}
\end{table}





\subsubsection{Ablation Study on Low-Rank Dimension}
\label{sec:ablation_low_rank_dimension}

Figure~\ref{fig:ablation_low_rank_dimension} illustrates the Stage-I alignment dynamics across varying low-rank dimensions on Wan2.1-T2V-1.3B. While all configurations converge reliably, aggressively compressed bottlenecks (e.g., $r=8$) exhibit higher loss floors and optimization variance, indicating constrained representational capacity. Conversely, expanding the bottleneck dimension monotonically improves approximation fidelity and convergence stability. Specifically, \textbf{$r=64$ achieves the lowest MSE loss and optimal optimization behavior, demonstrating a superior capacity to decode the continuous background manifold}. Consequently, we adopt $r=64$ as our default structural prior, balancing extreme parameter efficiency with high-fidelity context restoration.

\begin{figure}[t]
    \centering
    \includegraphics[width=0.82\linewidth]{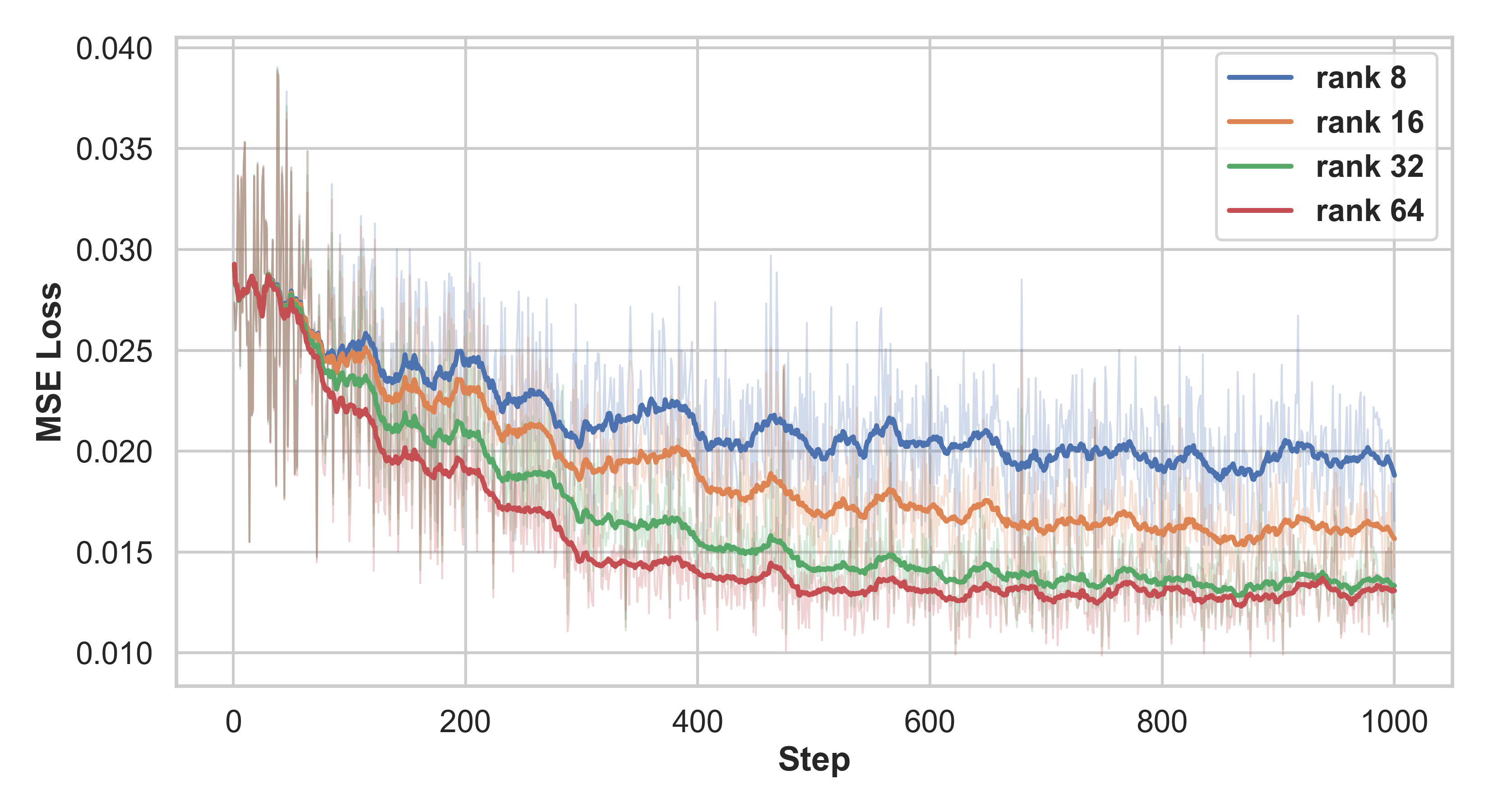}
    \caption{Training MSE loss under different low-rank dimensions.}
    \label{fig:ablation_low_rank_dimension}
\end{figure}

\subsubsection{Ablation Study on Activation Function}
\label{subsec:act_ablation}
We further conduct an ablation study on the activation function used in the low-rank compensator. In this experiment, we keep the sparse branch fixed as VMoBA and only replace the non-linear activation in the low-rank compensator, so that the effect of different activation choices can be isolated. 

As shown in Table~\ref{tab:activation_functions_low_rank}, Sigmoid achieves the best overall performance among the tested activation functions. It obtains the highest scores on SC, BC, MS, and AQ, while maintaining a competitive IQ score. This indicates that a bounded and smooth activation is more suitable for the low-rank compensator, helping it provide stable complementary information to the sparse attention branch.

\begin{table}[H]
\centering
\caption{Comparison of Activation Functions in the low-rank compensator}
\label{tab:activation_functions_low_rank}
\begin{tabular}{lccccc}
\toprule
\textbf{Activation Function} & \textbf{SC$\uparrow$} & \textbf{BC$\uparrow$} & \textbf{MS$\uparrow$} & \textbf{AQ$\uparrow$} & \textbf{IQ$\uparrow$} \\
\midrule
ReLU    & 0.9353 & 0.9346 & 0.9871 & 0.5689 & 0.6458 \\
GELU    & 0.9371 & 0.9395 & 0.9884 & 0.5906 & 0.6291 \\
Tanh    & 0.9433 & 0.9414 & 0.9854 & 0.5421 & \textbf{0.6465} \\
\textbf{Sigmoid} & \textbf{0.9510} & \textbf{0.9487} & \textbf{0.9898} & \textbf{0.6071} & 0.6440 \\
\bottomrule
\end{tabular}
\end{table}

\subsubsection{Ablation Study on Training Schedule Robustness}
\label{sec:Ablation Study on Training Schedule}
To validate the stability of our two-stage post-training pipeline, we analyze the loss dynamics under varying Stage-I alignment durations (ranging from 400 to 1000 steps). As illustrated in Figures \ref{fig:stage1_loss_curve} and \ref{fig:total_loss_curve}, the Stage-I MSE loss reliably converges to a shared minimum regardless of the scheduled length, \textbf{ensuring sufficient structural alignment of the newly introduced low-rank compensator}. Furthermore, upon transitioning to Stage-II full fine-tuning, the optimization trajectories swiftly stabilize and converge to comparable optimal levels across all settings. \textbf{These observations conclusively demonstrate that our decoupled training strategy is robust and insensitive to the specific transition threshold}.

\begin{figure}[t]
    \centering
    \includegraphics[width=0.7\linewidth]{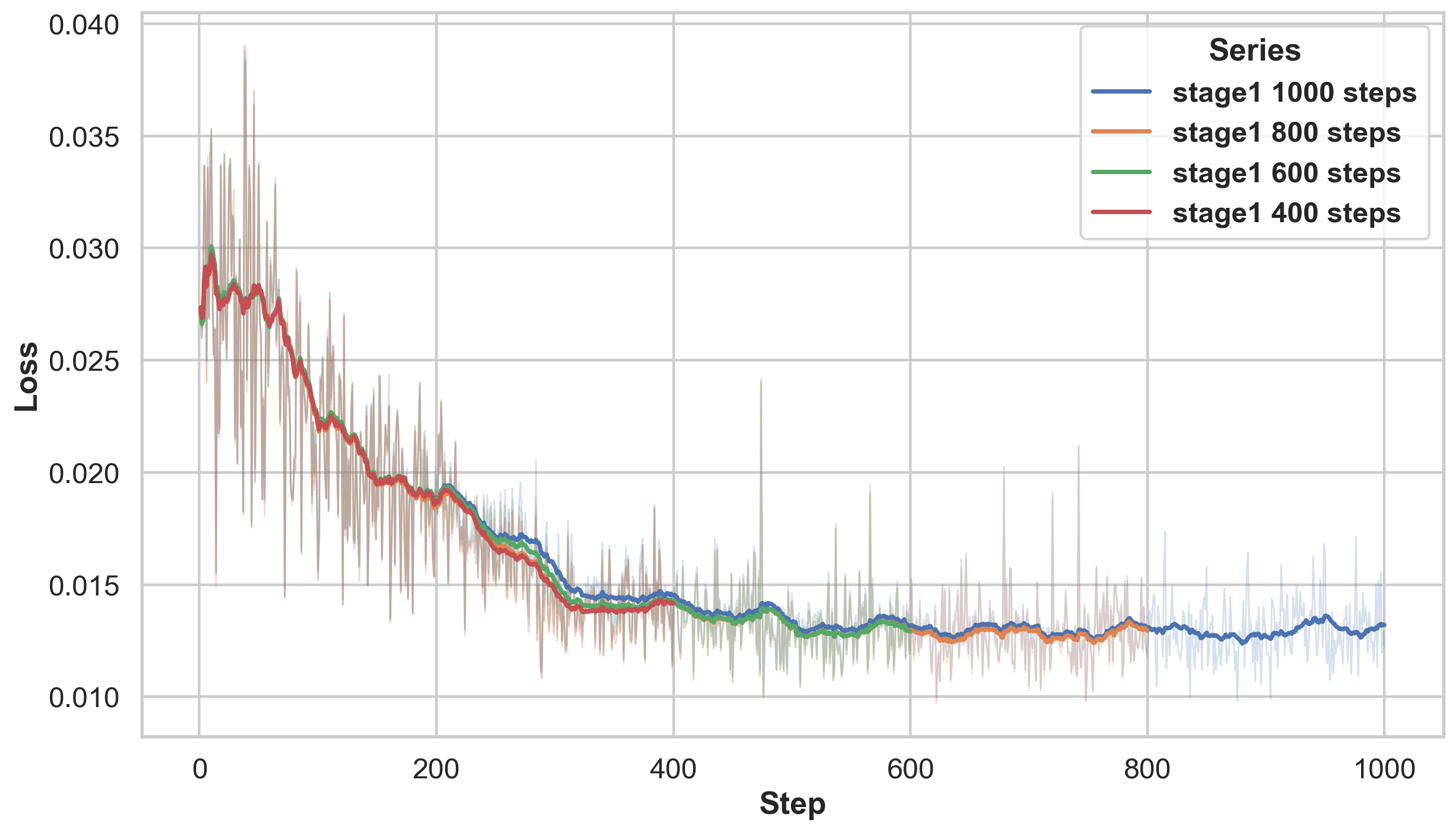}
    \caption{Stage-I training loss on Wan2.1-T2V-1.3B under different stage-I lengths (400, 600, 800, and 1000 steps).}
    \label{fig:stage1_loss_curve}
\end{figure}

\begin{figure}[t]
    \centering
    \includegraphics[width=0.7\linewidth]{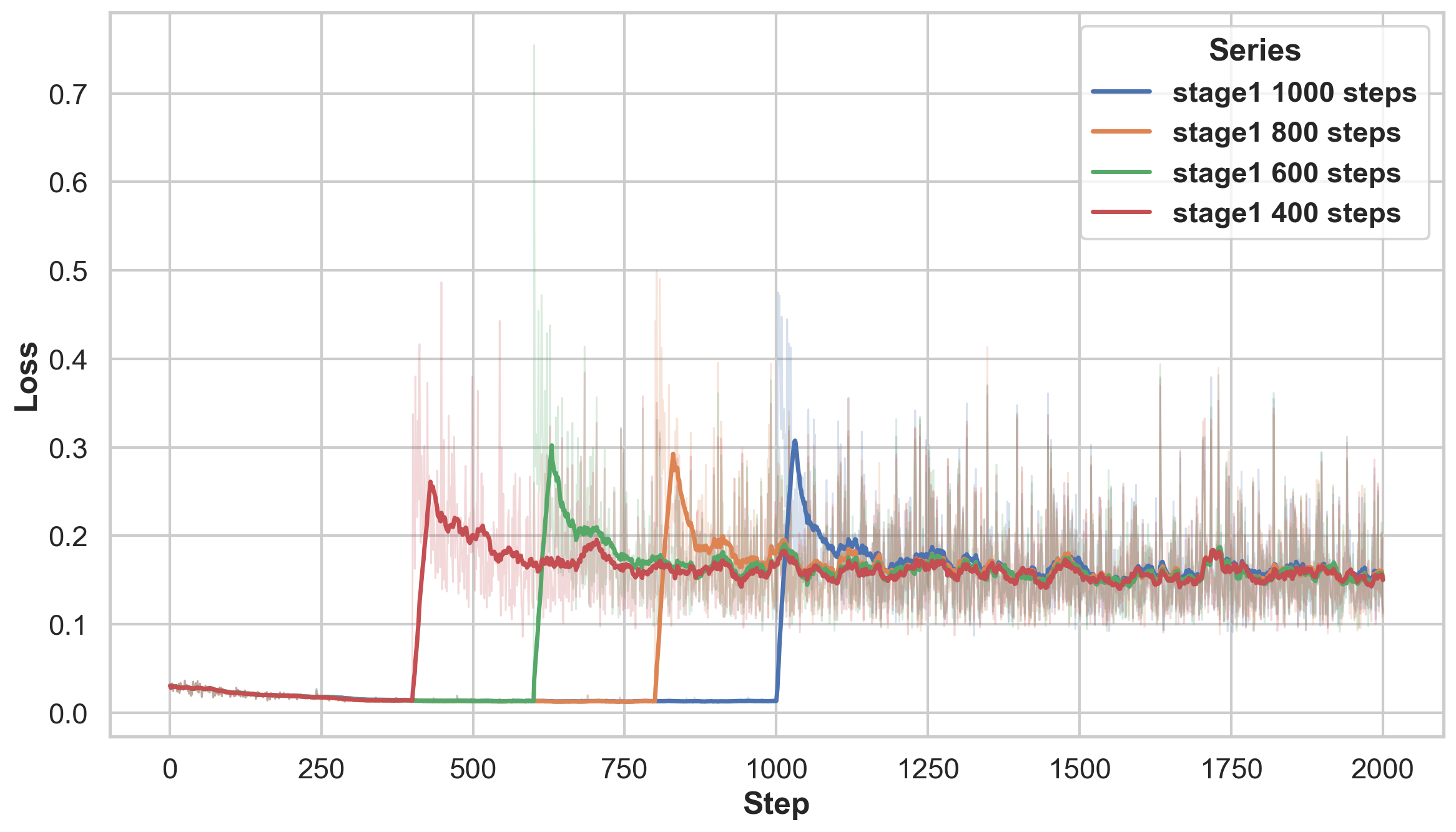}
    \caption{Overall training loss on Wan2.1-T2V-1.3B under different stage-I lengths (400, 600, 800, and 1000 steps). The loss increase near the transition point is caused by the change of training objective between stage I and stage II.}
    \label{fig:total_loss_curve}
\end{figure}

\subsection{Sparsity-Aligned Evaluation of SVG2}
We perform a sparsity-aligned evaluation on SVG2. Since the original implementations of SVG2 cannot precisely control the exact sparsity ratio, all reported results in this section are approximate values with a sparsity error no larger than 1\%. To further demonstrate the superiority of our RoPeSLR, the evaluation is conducted by aligning the sparsity of SVG2 to 90\%, consistent with the sparsity setting of RoPeSLR. Results in Table~\ref{tab:svg2_vsa_sparsity_align} show that even after sparsity alignment, our RoPeSLR still comprehensively outperforms SVG2 in core VBench quality metrics and efficiency, further validating the superiority of our proposed method.

\begin{table}[H]
\centering
\fontsize{9}{10}\selectfont   
\setlength{\tabcolsep}{2.5pt}
\caption{Sparsity-Aligned Evaluation of SVG2.}
\label{tab:svg2_vsa_sparsity_align}
\begin{tabular}{lcccccc}
\hline
Method & Sparsity & SC & MS & BC & IQ & AQ \\
\hline
SVG2 & 90\% & 0.8980 & 0.9701 & 0.9265 & 0.5877 & 0.4988 \\
RoPeSLR (Ours) & 90\% & \textbf{0.9510} & \textbf{0.9898} & \textbf{0.9487} & \textbf{0.6440} & \textbf{0.6071} \\
\hline
\end{tabular}
\end{table}

\subsection{Quantitative results of Wan2.1-T2V-14B}

Table~\ref{tab:wan14b_benchmark} reports the quantitative results on the Wan2.1-T2V-14B model. 
We conduct inference using 81-frame videos at a resolution of $720 \times 1280$, and evaluate both generation quality and attention FLOPs under the same sampling setting. 
Compared with Wan2.1-T2V-1.3B, this setting further examines whether RoPeSLR can scale to a substantially larger video generation backbone and higher-resolution generation scenario.

The quantitative results show that RoPeSLR achieves a favorable quality--efficiency trade-off on Wan2.1-T2V-14B. Compared with existing sparse baselines, RoPeSLR is expected to better preserve generation quality while maintaining high computational efficiency. This suggests that the low-rank compensator effectively restores global context under extreme sparsity, enabling RoPeSLR to retain strong visual fidelity and semantic consistency on large-scale video generation models.

\vspace{-0.5cm}
\begin{table*}[htbp]
  \centering
  \fontsize{7}{8}\selectfont  
  \setlength{\tabcolsep}{2.1pt} 
  \caption{Quantitative comparison on Wan2.1-14B benchmark.}
  \label{tab:wan14b_benchmark}
  \renewcommand{\arraystretch}{1.0}
  \begin{tabular}{llccccccc}
    \toprule
    \multirow{2}{*}{\textbf{Model}} &
    \multirow{2}{*}{\textbf{Method}} &
    \multicolumn{5}{c}{\textbf{VBench Quality}} &
    \multicolumn{2}{c}{\textbf{Efficiency}} \\
    \cmidrule(lr){3-7} \cmidrule(lr){8-9}
    & &
    \textbf{SC↑} &
    \textbf{MS↑} &
    \textbf{BC↑} &
    \textbf{IQ↑} &
    \textbf{AQ↑} &
    \textbf{FLOPs↓} &
    \textbf{Sparsity} \\
    \midrule

    \multirow{5}{*}{Wan2.1-14B}
    & Full  & \textbf{0.9697} & \textbf{0.9884} & \textbf{0.9668} & \textbf{0.7006} & \textbf{0.6246} & 4682.02 TFLOPs & 0\%  \\
    & SVG2   & 0.9511 & 0.9833 & 0.9548 & 0.5978 & 0.5932 & 1655.77 TFLOPs & 64\% \\
    & VSA    & 0.8595 & 0.9849 & 0.9470 & 0.4959 & 0.5954 & \textbf{468.60 TFLOPs} & 90\%  \\
    & SLA    & 0.9479 & 0.9860 & 0.9489 & 0.6468 & 0.6021 & 480.09 TFLOPs & 90\%  \\
    & \cellcolor{blue!10}\textbf{RoPeSLR (Ours)}
    & \cellcolor{blue!10} 0.9615 
    & \cellcolor{blue!10} 0.9868                   
    & \cellcolor{blue!10} 0.9546 
    & \cellcolor{blue!10} 0.6624                   
    & \cellcolor{blue!10} 0.6139 
    & \cellcolor{blue!10} 472.19 TFLOPs
    & \cellcolor{blue!10}\textbf{90\%} \\ 
    \bottomrule
  \end{tabular}
\end{table*}

\subsection{Efficiency Analysis}
\label{subsec:efficiency}
\subsubsection{Efficiency Analysis of the Low-Rank compensator}
We analyze the computational overhead of different approximation methods for the lowrank compensator (bottom-90\%), comparing our low-rank auxiliary compensator against two baseline implementations: full softmax attention and linear attention.

\begin{figure}[H]
  \centering
  \includegraphics[width=0.7\linewidth]{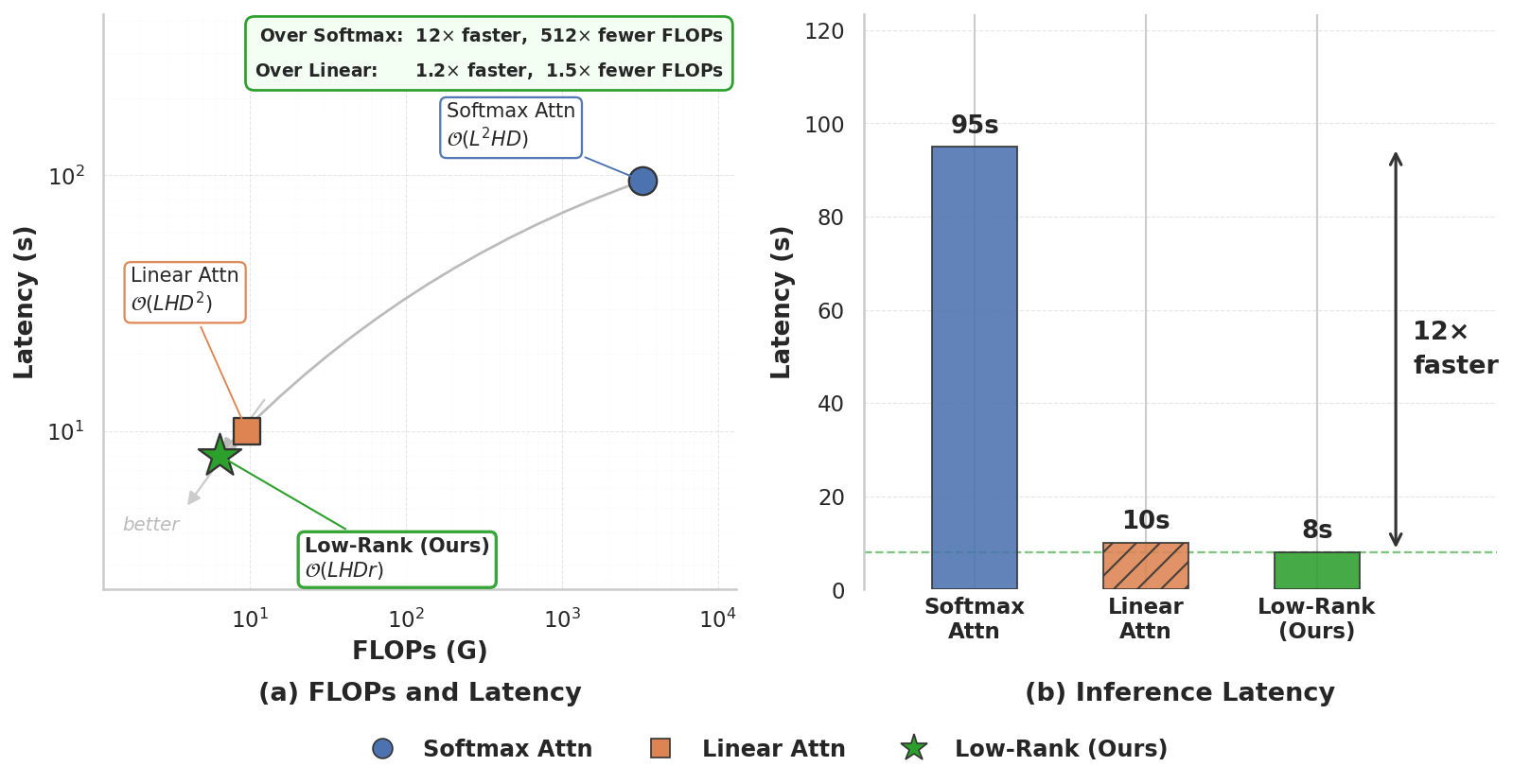}
  \caption{FLOPs and latency comparison of attention approximation methods for the non-top-k branch on Wan2.1-1.3B. (a) FLOPs-latency trade-off on a log-log scale; (b) absolute inference latency per forward pass.}
  \label{fig:lowrank_efficiency}
\end{figure}

As shown in Figure~\ref{fig:lowrank_efficiency},  RoPESLR delivers the best overall efficiency among the three methods. Compared to full softmax attention, it reduces FLOPs by $256\times$ and end-to-end inference latency by $12\times$. It also outperforms linear attention, delivering $1.25\times$ faster inference with lower computational overhead.

\subsubsection{Efficiency Analysis of Latency for Long-Sequence Video Tokens}

We further evaluate the latency of RoPeSLR on long-sequence video generation using HunyuanVideo-13B. The inference is conducted with 129 frames at $720 \times 1280$ resolution, corresponding to $33 \times 45 \times 80 = 118{,}800$ video tokens after VAE compression and patchification. We compare the end-to-end inference latency of the full-attention baseline and RoPeSLR under the same inference setting, as shown in Table~\ref{tab:long_sequence_latency}.

\begin{table*}[htbp]
\centering
\caption{End-to-end latency comparison on HunyuanVideo-13B with 118{,}800 long-sequence video tokens.}
\label{tab:long_sequence_latency}
\begin{tabular}{lcc}
\toprule
\textbf{Method} &
\textbf{Latency (s)$\downarrow$} &
\textbf{Speedup$\uparrow$} \\
\midrule
Full Attention & 3354s & -- \\
RoPeSLR (Ours) & \textbf{1482s} & \textbf{2.26$\times$} \\
\bottomrule
\end{tabular}
\end{table*}

\subsection{Sparse and Low-Rank Characteristics of LLM Attention}
\label{subsec:llm_attn_char}

We analyze the causal self-attention weight matrices from Qwen2-7B-Instruct (28 GQA layers, 28 query heads per layer). For each row of the attention matrix, we retain the entries with the top 90\% cumulative energy and exclude them, then analyze the remaining low-energy residual part. Rows with fewer than 128 visible keys are excluded, resulting in 784 valid (layer, head) combinations.

As illustrated in Figure~\ref{fig:attn_decomp_qwen}, the left energy concentration curve presents a steeper rising trend at the initial segment compared with video diffusion Transformers, which is mainly attributed to the inherent attention sink effect in LLMs. The right subfigure presents the stable rank distribution of low-energy residuals, with a median stable rank around 40. Compared with video DiT, the substantially higher stable rank of LLM residuals indicates that \textbf{LLMs do not exhibit the prominent sparse and low-rank structural patterns observed in video diffusion models} (see Figure~\ref{fig:attn_decomp_hunyuan}).

\begin{figure}[htbp]
\centering
\includegraphics[width=0.8\textwidth]{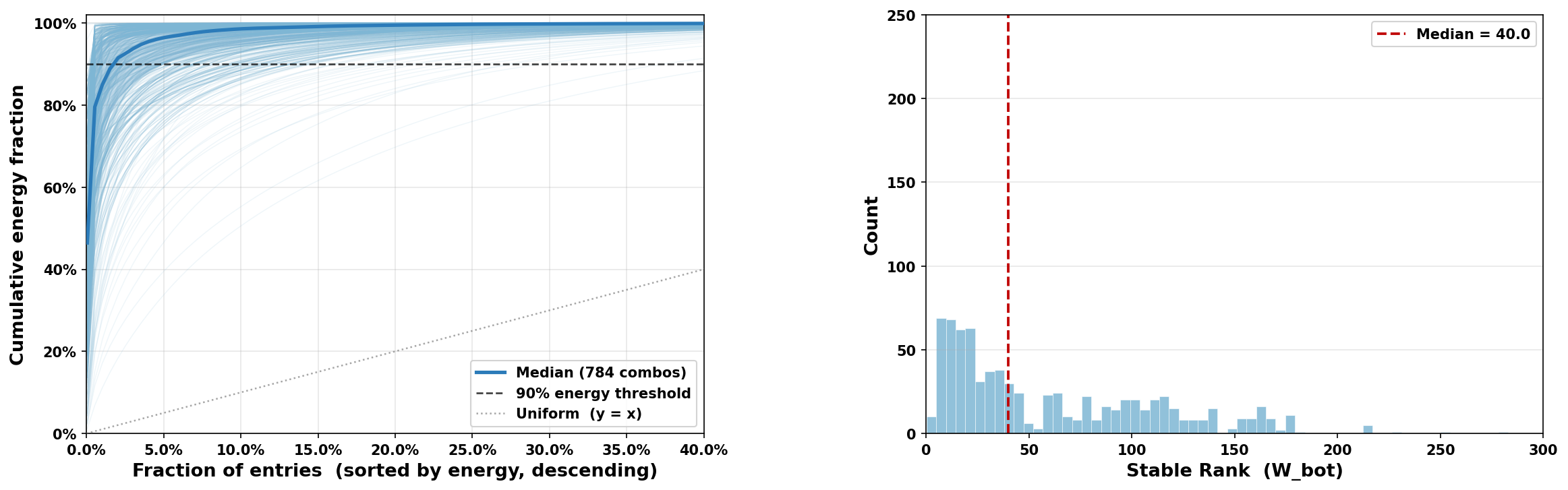}
\caption{
Analysis of causal self-attention matrices from Qwen2-7B-Instruct under a 90\% energy threshold over 784 valid combinations.
}
\label{fig:attn_decomp_qwen}
\end{figure}

\subsection{Visualization of Learned Gate Distributions}
\label{subsec:gate_vis}

We visualize the gate activation intensity across denoising timesteps to examine
whether the gating mechanism captures content-dependent spatial structure.
As shown in Figure~\ref{fig:gate_spatial}, we present the gate activation maps
for five representative layers at ten uniformly sampled denoising timesteps,
alongside a reference frame from the generated video.
Two consistent patterns can be observed.

\textbf{Spatial structure in gate activation aligns with scene content.}
The gate activation maps are not spatially uniform but exhibit structural contours
that correspond to the scene layout visible in the reference frame:
the airplane body, the sky--ground boundary, and the grass--runway edge are all
visibly delineated in the gate intensity maps across all five layers.
This indicates that the gate responds to the semantic content of the video
rather than assigning uniform weight across spatial positions.

\textbf{Spatial structure becomes more pronounced in later denoising steps.}
In early denoising steps ($t \leq 11$), gate activation is relatively uniform
across spatial positions.
As denoising progresses ($t \geq 22$), the spatial contours become increasingly
distinct, with higher gate activation concentrated around semantically meaningful
regions such as the airplane body and scene boundaries.
This evolution suggests that the gate progressively adapts to the emerging
content structure as the video is refined through the denoising process.

\begin{figure}[h]
    \centering
    \begin{minipage}[c]{0.12\linewidth}
        \centering
        \includegraphics[width=\linewidth]{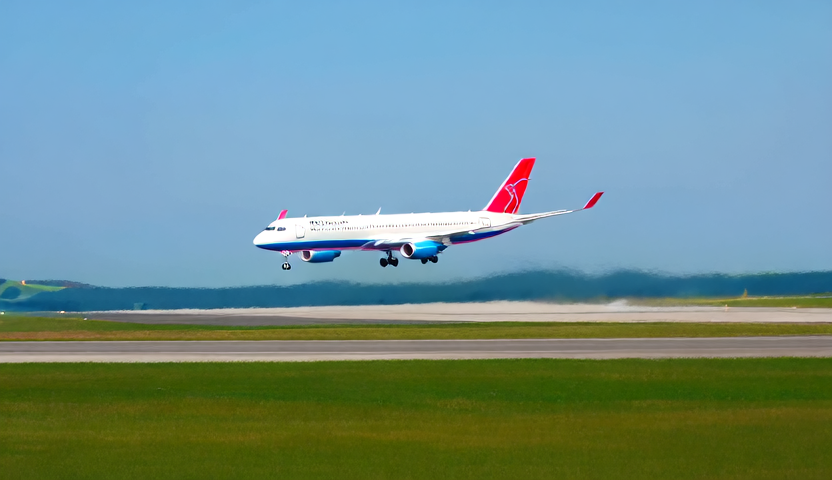}
       \vspace{0.5em} 
        \small\textit{Reference frame}
    \end{minipage}
    \hfill
    \begin{minipage}[c]{0.86\linewidth}
        \centering
        \includegraphics[width=0.8\linewidth]{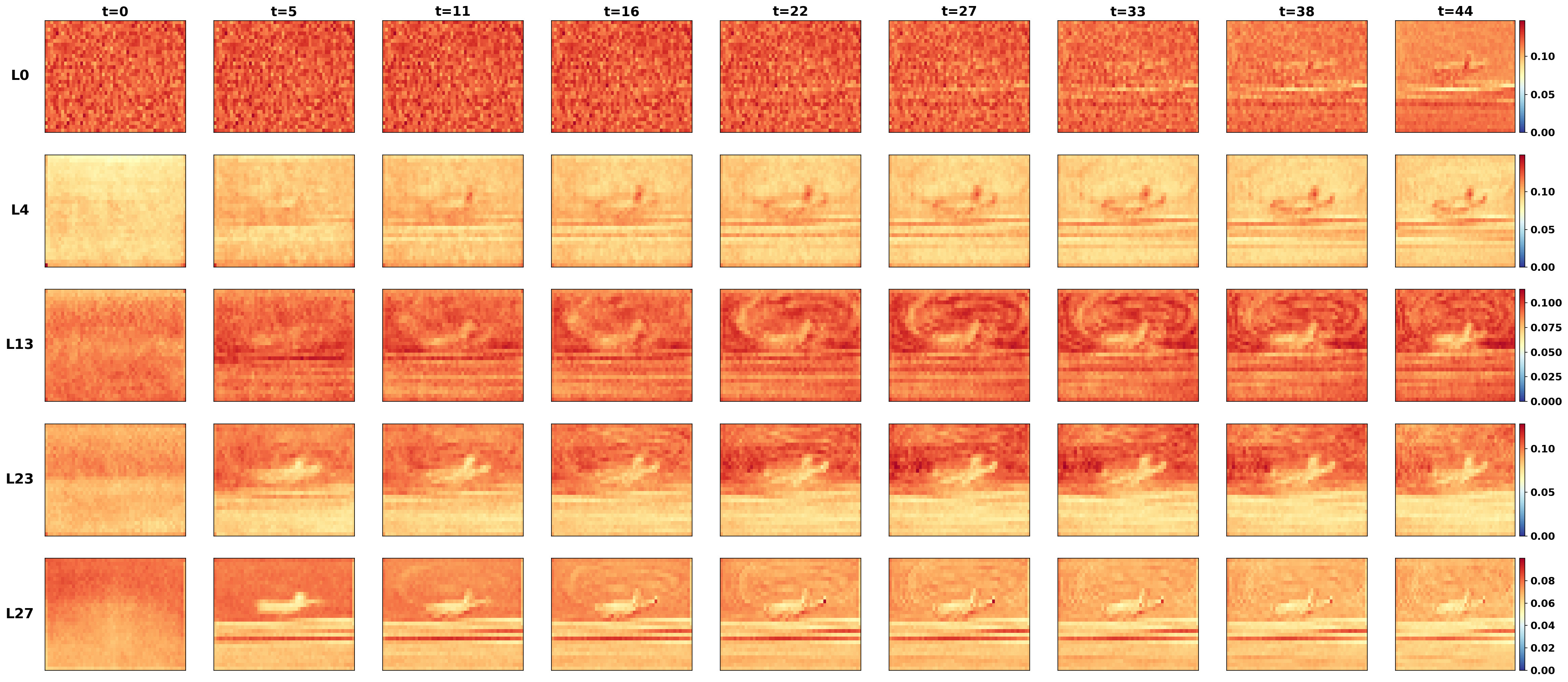}
    \end{minipage}
    \caption{
        Gate activation intensity maps of our RoPeSLR on Wan2.1-T2V-1.3B (480P),
        shown for five representative layers (L0, L4, L13, L23, L27) alongside a
        reference frame from the generated video
        (prompt: \textit{``an airplane taking off''}).
        Each row corresponds to one transformer layer; each column is a uniformly
        sampled denoising timestep ($t{=}0$ to $t{=}49$, 10 steps).
        Each cell shows the mean gate value over all attention heads, averaged over
        $T{=}21$ latent frames and projected onto the $30{\times}52$ spatial latent grid.
        Each row uses an independent color scale.
    }
    \label{fig:gate_spatial}
\end{figure}

\subsection{Mechanistic Interpretability via Gram Spectral Analysis}
\label{subsec:svd_mechanistic_analysis}

To mechanistically validate the algebraic insights in \textbf{Remark \ref{remark:token_wise_derivation}}, we perform Singular Value Decomposition (SVD) on the unnormalized sequence-level outputs of the rank-64 compensator, $O_{\text{lr}} \in \mathbb{R}^{L \times d_h}$. By factoring $O_{\text{lr}} = U \Sigma V^\top$, we extract the singular values $\sigma_i$ (which quantify the energy or variance captured by each mode) and the left singular vectors $\mathbf{u}_i \in \mathbb{R}^L$ (which represent the spatial principal components). Reshaping $\mathbf{u}_i$ back to the native spatiotemporal grid allows us to directly visualize the learned coordinate mappings.

\textbf{Eigenspectrum Analysis: Spectral Collapse and Structural Alignment.}
The normalized eigenspectrum (\textbf{Figure \ref{fig:svd_spectrum}, Top}) plots the energy distribution across the principal components. Without 3D positional priors, the compensator suffers from severe \textit{spectral collapse}: its spectrum degenerates into a flat noise floor, demonstrating an inability to express meaningful spatial variations. Conversely, the explicit injection of 3D PE causally restores this spatial variance. The resulting spectrum tracks the Full Attention oracle, exhibiting a exponential decay. Notably, the sharp energy drop near the truncation index ($r=64$) provides strong empirical validation for the rank choice of RoPeSLR.

\textbf{Spatial Mode Analysis: Pure Decoupling and Low-pass Filtering.}
Visualizing the reshaped dominant spatial modes (\textbf{Figure \ref{fig:svd_spectrum}, Bottom}) exposes the underlying manifold decoupling mechanism. The principal components of exact Full Attention are visibly entangled with high-frequency, localized semantic spikes (e.g., sharp object outlines). In stark contrast, our 3D PE-equipped compensator completely filters out these local semantics, synthesizing highly pure, low-frequency geometric standing waves. This provides definitive evidence that the low-rank compensator acts as a \textit{spatial low-pass filter}. It flawlessly decodes the smooth background continuum, explicitly leaving all high-frequency semantic routing to the complementary sparse branch.

\begin{figure}[h]
    \centering
    \includegraphics[width=0.9\textwidth]{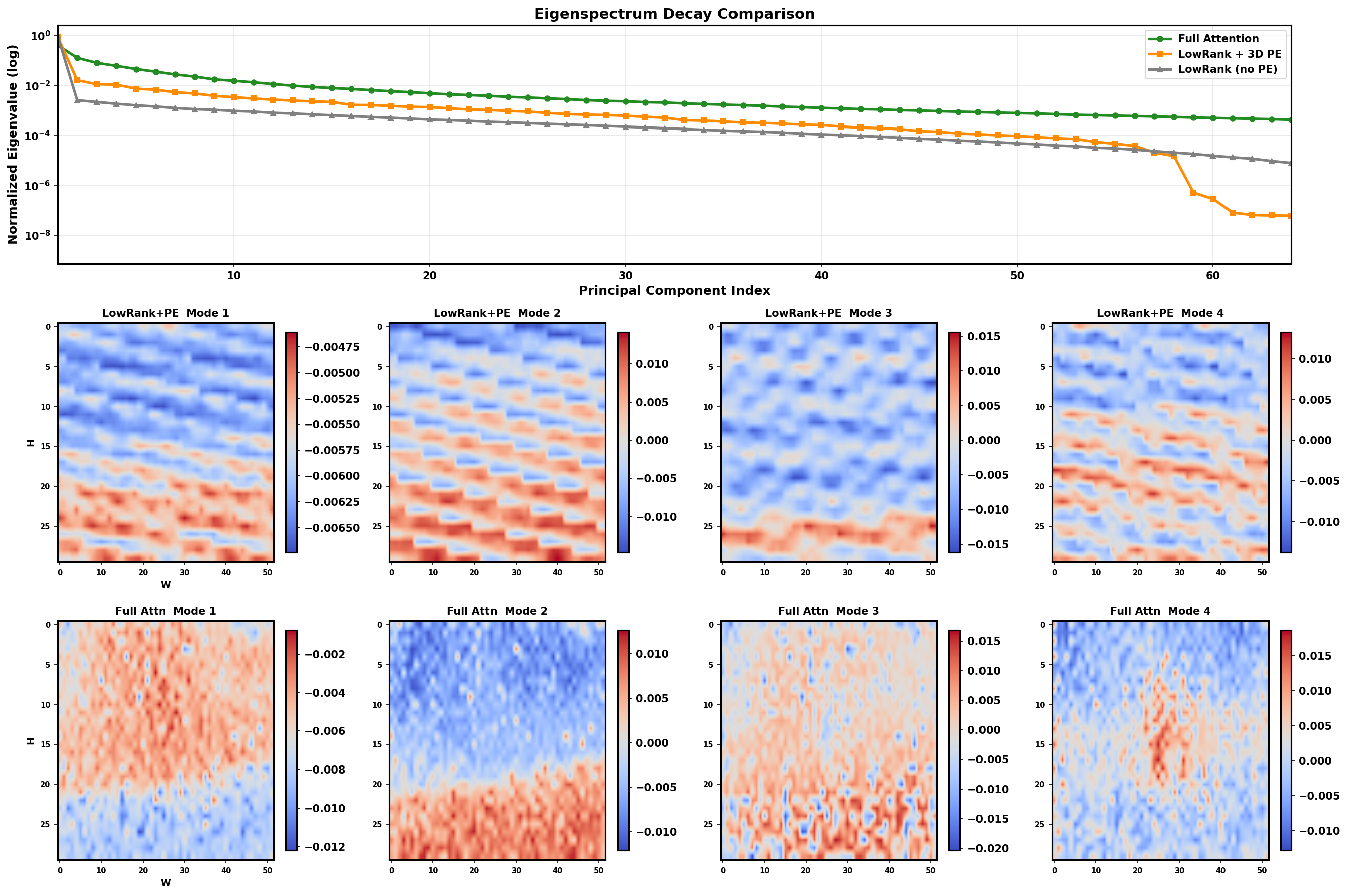}
    \caption{
    \textbf{Mechanistic Interpretability via Gram Spectral Analysis.} 
    \textit{Methodology:} Evaluated on Wan2.1-T2V-1.3B (Layer 15, Denoising Step 10) at $480 \times 832$ resolution. SVD is applied to the pre-activation output $O_{\text{lr}}$ over a latent grid of $18 \times 30 \times 52$ ($L=28,080$). 
    \textit{Analysis:} \textbf{(Top)} The eigenspectrum (energy distribution) shows that 3D PE induces an exponential decay aligning with Full Attention, preventing the spectral collapse seen in the no-PE baseline. The cliff-drop near $r=64$ confirms the theoretical rank bounds. \textbf{(Bottom)} Heatmaps of the top-4 spatial principal components ($\mathbf{u}_1$ to $\mathbf{u}_4$) at center frame $t=9$. Unlike the noisy, high-frequency semantics entangled in Full Attention, our compensator generates clean geometric standing waves, validating its role as a spatial low-pass filter.
  }
    \label{fig:svd_spectrum}
\end{figure}

\subsection{Compatibility with Step-Distilled Models}
\label{app:distillation}

Our method is orthogonal to various step distillation techniques and can be directly deployed on distilled models to achieve compounded acceleration. Step distillation reduces the number of diffusion denoising steps, while our method significantly lowers the attention computation cost per step.

\paragraph{Why Sparse Routing Remains Effective}
RoPeSLR leverages the inherent spatiotemporal sparsity of natural video signals through chunk-based sparse routing. Due to the strong redundancy across consecutive video frames, attention energy naturally concentrates in local spatiotemporal neighborhoods. This structural property is determined solely by input data characteristics, rather than the number of denoising steps.

Step-distilled models require a larger denoising jump in each inference step, which further concentrates attention energy onto informative regions. This improves the efficiency of sparse selection, requiring fewer chunks to reach the predefined energy threshold $\tau$.

\paragraph{Transferability of the Low-Rank Branch}
The low-rank compensator of RoPeSLR is formulated as $\sigma(\sigma(\hat{X} W_A) W_B)$, which learns a compact mapping within the hidden representation space. Step distillation generally preserves the original model architecture, including hidden dimension, number of attention heads, and per-head dimension. Since the representation space remains unchanged, the learned low-rank approximation stays valid.

\paragraph{Architectural Compatibility Analysis}
A key characteristic of step distillation is that it generally maintains the core architecture of the original model. It mainly optimizes the denoising step count without modifying the fundamental model structure. Specifically, step distillation does not alter the self-attention configuration, including network depth, number of attention heads, per-head dimension $d_\text{head}$, and the projection dimensions of Q/K/V, thus preserving the integrity and consistency of the self-attention pathway. Furthermore, RoPeSLR can also be fine-tuned on distilled models to achieve better performance.

\newpage

\section{Algorithm and Hyperparameters}
\label{subsec:hyperparams}
\begin{algorithm}[H]
\caption{RoPeSLR Attention Mechanism}
\label{alg:ropeslr}
\begin{algorithmic}[1]
    \Require Input sequence $X \in \mathbb{R}^{L \times d_{model}}$, 3D coordinates $(t, x, y)$, energy threshold $\tau$, low-rank dimension $r$, and learnable parameters $\Theta$.
    \Ensure Refined representation $O \in \mathbb{R}^{L \times d_{model}}$.
    
    \Statex \textit{// Step 1: Positional Context Recovery}
    \State Construct 3D Absolute PE using trigonometric bases: $\text{PE}_{3\text{D}} = f(t, x, y)$;
    \State Inject geometric priors: $\hat{X} = X + \alpha \odot \text{PE}_{3\text{D}}$; \Comment{Eq. \ref{eq:3D PE+X} in text}
    
    \Statex \textit{// Step 2: Dual-Branch Routing (Per Head $h$)}
    \For{each attention head $h \in \{1, \dots, H\}$}
        \State \textbf{Sparse Branch (Energy-Driven):}
        \State Extract localized high-energy spikes guided by $\tau$: $O_{sparse}^{(h)} = \text{SparseAttn}(X, \mathcal{N}(p); \tau)$; \Comment{via VMoBA\cite{wu2025vmoba}}
        
        \State \textbf{Low-Rank Compensator:}
        \State Perform implicit context decoding: $O_{lowrank}^{(h)} = \sigma\!\left(\sigma(\hat{X} W_A^{(h)}) W_B^{(h)}\right)$; \Comment{$W_A^{(h)} \!\in\! \mathbb{R}^{d_{head} \times r}$, $W_B^{(h)} \!\in\! \mathbb{R}^{r \times d_{head}}$}

    \EndFor
    
    \Statex \textit{// Step 3: Distribution Alignment and Gated Fusion}
    \State Stabilize outputs: $\tilde{O}_{sparse} = \text{RMSNorm}(O_{sparse})$, $\tilde{O}_{lowrank} = \text{RMSNorm}(O_{lowrank})$;
    \State Compute adaptive gating: $g = \sigma(\hat{X} W_g + b_g)$; \Comment{Token-wise gate \ref{eq:gate}}
    \State \textbf{Final Synthesis:} $O = \tilde{O}_{sparse} + g \odot \tilde{O}_{lowrank}$; \Comment{Eq. \ref{eq: output combination} in text}
    
    \State \Return $O$
\end{algorithmic}
\end{algorithm}

Table~\ref{tab:hyperparams} lists the main hyperparameters used for fine-tuning.

\begin{table}[htbp]
\centering
\caption{Training and Model Hyperparameters for our RoPeSLR on Wan2.1-T2V-1.B}
\label{tab:hyperparams}
\begin{tabular}{lc}
\toprule
\textbf{Hyperparameter} & \textbf{Value} \\
\midrule
Optimizer               & AdamW \\
$\beta_1$, $\beta_2$    & 0.9, 0.99 \\
Weight decay & $10^{-5}$ \\
Max gradient norm       & 1.0 \\
Precision               & BF16 \\
Low-rank dimension $r$  & 64 \\
\midrule
\multicolumn{2}{l}{\textit{Two-Stage Training Schedule}} \\
Stage I loss    & MSE  \\
Stage II loss    & Flow Matching Loss \\
Stage I LR schedule     & Constant \\
Stage II LR schedule     & Cosine \\
\midrule
\multicolumn{2}{l}{\textit{Learning Rates (Stage I / Stage II)}} \\
Backbone                & — / $10^{-5}$ \\
Low-rank projections (proj\_a/b) & $3\times10^{-4}$ / $10^{-4}$ \\
Gate projection         & $3\times10^{-4}$ / $10^{-4}$ \\
Gate bias               & $3\times10^{-4}$ / $10^{-4}$ \\
PE scale                & $3\times10^{-4}$ / $10^{-4}$ \\
\midrule
\bottomrule
\end{tabular}
\end{table}

\newpage
\section{Computational Complexity Analysis}
\label{sec:complexity_analysis}

To evaluate the efficiency of RoPeSLR, we provide a detailed theoretical analysis of the floating-point operations (FLOPs) required during the forward pass. Based on our network implementation, we prove two core conclusions:
\begin{itemize}[leftmargin=*, itemsep=0em]
    \item Under identical sparsity, RoPeSLR is strictly more computationally efficient than existing sparse-plus-linear attention frameworks;
    \item Compared to pure sparse attention, the extra computational overhead introduced by our low-rank compensator is asymptotically negligible for long sequences.
\end{itemize}

\subsection{Notation and Basic Definitions}
Let us define the architectural dimensions and hyperparameters as follows:
\begin{itemize}[leftmargin=*, itemsep=0em]
    \item $B$: Batch size.
    \item $L$: Spatiotemporal sequence length.
    \item $H$: Number of attention heads.
    \item $d_h$: Feature dimension per head, yielding a total hidden dimension of $d_{\text{model}} = H \times d_h$.
    \item $S$: Sparsity ratio of the sparse branch (the proportion of masked tokens). The active computation ratio is $(1-S)$.
    \item $r$: The uniform low-rank bottleneck dimension applied across all attention heads.
\end{itemize}
Note: For standard matrix multiplication between $X \in \mathbb{R}^{M \times N}$ and $Y \in \mathbb{R}^{N \times P}$, the required FLOPs are $2MNP$. Memory-bound operations (e.g., element-wise addition for 3D PE injection, RMSNorm, and Sigmoid activations) scale linearly with the tensor size $\mathcal{O}(BHLd_h)$ and are structurally negligible compared to the dominant multiply-accumulate (MAC) FLOPs.

\subsection{Complexity Anatomy of RoPeSLR}
The attention computation within a single RoPeSLR transformer block can be decoupled into three primary components:

\textbf{Adaptive Sparse Branch.}
For block-sparse attention (e.g., VMoBA) with a sparsity ratio of $S$, each query token attends to $L(1-S)$ key tokens. The computation involves the $QK^\top$ similarity and the subsequent $\text{Softmax}(\cdot)V$ aggregation. The total FLOPs are:
\begin{equation}
    \mathcal{C}_{\text{sparse}} = \sum_{h=1}^H \left( 2 \cdot B \cdot L \cdot [L(1-S)] \cdot d_h + 2 \cdot B \cdot L \cdot [L(1-S)] \cdot d_h \right) = 4 B H L^2 (1-S) d_h
\end{equation}

\textbf{Head-wise Low-Rank Compensator.}
Unlike standard linear attention, which computes a global sequence-level $K^\top V$ state explicitly, we parameterize the global context implicitly via a head-wise two-layer MLP: $\hat{X}^{(h)} \to W_A^{(h)} \to \sigma \to W_B^{(h)} \to \sigma$. For the input slice $\hat{X}^{(h)} \in \mathbb{R}^{B \times L \times d_h}$ corresponding to head $h$:
\begin{itemize}[leftmargin=*, itemsep=0em]
    \item Down-projection $\hat{X}^{(h)} W_A^{(h)}$, where $W_A^{(h)} \in \mathbb{R}^{d_h \times r}$. FLOPs: $2 B L d_h r$.
    \item Up-projection $\sigma(\cdot) W_B^{(h)}$, where $W_B^{(h)} \in \mathbb{R}^{r \times d_h}$. FLOPs: $2 B L r d_h$.
\end{itemize}
Summing over all $H$ heads, the computational cost of the uniform low-rank compensator is:
\begin{equation}
    \mathcal{C}_{\text{lowrank}} = \sum_{h=1}^H \left( 2 B L d_h r + 2 B L r d_h \right) = 4 B H L d_h r
\end{equation}

\textbf{Gating and Fusion.}
This includes the gating projection ($W_g \in \mathbb{R}^{d_{\text{model}} \times 1}$) and the final gated convex combination:
\begin{itemize}[leftmargin=*, itemsep=0em]
    \item Gating projection $\hat{X} W_g$: $2 B L d_{\text{model}} \cdot 1 = 2 B H L d_h$ FLOPs.
\end{itemize}
\begin{equation}
    \mathcal{C}_{\text{fusion}} = 2 B H L d_h 
\end{equation}

Thus, the total MAC complexity of RoPeSLR is:
\begin{equation}
    \mathcal{C}_{\text{RoPeSLR}} = 4 B H L^2 (1-S) d_h + 4 B H L d_h r + 2 B H L d_h
\end{equation}

\subsection{Theoretical Superiority over Sparse-Linear Frameworks}
State-of-the-art fusion methods (e.g., SLA) adopt a "Sparse + Linear Attention" paradigm. The linear attention branch utilizes a feature map $\phi(\cdot): \mathbb{R}^{d_h} \to \mathbb{R}^{d_h}$ and computes the attention via matrix associativity:
\begin{enumerate}[leftmargin=*, itemsep=0em]
    \item \textbf{Global Context Aggregation.} $KV_{\text{ctx}} = \phi(K)^\top V$. Here, $\phi(K), V \in \mathbb{R}^{L \times d_h}$. The FLOPs are $2 B H L d_h^2$.
    \item \textbf{Context Retrieval.} $O_{\text{lin}} = \phi(Q) KV_{\text{ctx}}$. Here, $\phi(Q) \in \mathbb{R}^{L \times d_h}$ and $KV_{\text{ctx}} \in \mathbb{R}^{d_h \times d_h}$. The FLOPs are $2 B H L d_h^2$.
\end{enumerate}
Therefore, the baseline linear attention branch fundamentally imposes an $\mathcal{O}(d_h^2)$ bottleneck per token:
\begin{equation}
    \mathcal{C}_{\text{linear\_branch}} = 4 B H L d_h^2
\end{equation}
Assuming identical sparse branches and gating mechanisms, the efficiency disparity between RoPeSLR and standard linear hybrids lies entirely in the global context approximation branches. The ratio of their computational costs is exactly:
\begin{equation}
    \frac{\mathcal{C}_{\text{lowrank}}}{\mathcal{C}_{\text{linear\_branch}}} = \frac{4 B H L d_h r}{4 B H L d_h^2} = \frac{r}{d_h}
\end{equation}

\textbf{Conclusion 1.} In modern video DiTs, the head dimension $d_h$ is typically large (e.g., $d_h=128$). As established in Theorem \ref{theorem: main}, the exponential spectral decay of 3D RoPE yields an extremely low stable rank for the background context. This structural prior allows us to employ a highly compact, uniform bottleneck rank $r \ll d_h$ (e.g., $r=64$) across all heads without sacrificing representation capacity. Consequently, $\mathcal{C}_{\text{lowrank}} \ll \mathcal{C}_{\text{linear\_branch}}$. This theoretically proves that \textbf{RoPeSLR is structurally more compute-efficient than existing sparse-plus-linear architectures at identical sparsity levels}.

\subsection{Negligible Overhead Compared to Pure Sparse Attention}
To address potential concerns regarding the computational overhead introduced by our low-rank compensator, we analyze the extra overhead ratio $\eta$ relative to the pure sparse attention baseline:
\begin{equation}
    \eta = \frac{\mathcal{C}_{\text{lowrank}} + \mathcal{C}_{\text{fusion}}}{\mathcal{C}_{\text{sparse}}} = \frac{4 B H L d_h r + 2 B H L d_h}{4 B H L^2 (1-S) d_h} = \frac{4 r + 2}{4 L (1-S)} = \frac{r + 0.5}{L (1-S)}
\end{equation}

\textbf{Conclusion 2.} In high-fidelity video generation, scaling up spatiotemporal resolution inherently leads to a cubic explosion in sequence length ($L \ge 10^4$). Even at extreme sparsity settings (e.g., $S=0.90$), the number of active tokens $L(1-S)$ remains firmly in the thousands. Crucially, whether the uniform rank $r$ scales strictly according to our logarithmic theoretical bound ($r = \mathcal{O}(\log L)$) or is set as a fixed small hyperparameter ($r = \mathcal{O}(1)$), its growth is overwhelmingly eclipsed by the linear denominator $L$. Taking the asymptotic limit as sequence length scales:
\begin{equation}
    \lim_{L \to \infty} \eta = 0
\end{equation}
This proves that \textbf{as the sequence length scales towards ultra-long videos, the additional computational overhead of the RoPeSLR compensator strictly converges to zero}. In empirical practice, this overhead accounts for less than 5\% of the total computation. Therefore, RoPeSLR successfully restores the critical global distance decay of 3D RoPE while perfectly preserving the extreme acceleration benefits of pure sparse attention.

\section{Broader Impacts}\label{sec: broader impacts}
This work introduces RoPeSLR, an efficient attention framework for Diffusion Transformers (DiTs). By significantly reducing the computational overhead and memory footprint of high-fidelity video generation, our method contributes to the development of "Green AI" by lowering energy consumption and carbon emissions associated with large-scale video inference. Furthermore, it democratizes access to most widely used video generation models, enabling researchers and practitioners with limited computational resources to deploy and build upon these powerful generative tools.

\newpage

\end{document}